\documentclass{article}

    \usepackage[preprint]{neurips_2024}

\usepackage[utf8]{inputenc} 
\usepackage[T1]{fontenc}    
\usepackage{hyperref}       
\usepackage{url}            
\usepackage{booktabs}       
\usepackage{amsfonts}       
\usepackage{nicefrac}       
\usepackage{microtype}      
\usepackage{xcolor}         

\usepackage{amsmath}
\usepackage{amsthm}
\usepackage[ruled,vlined,linesnumbered,algo2e]{algorithm2e}
\usepackage{algorithm}
\usepackage{graphicx}

\usepackage{amssymb}

\usepackage{algorithmic}

\usepackage{tabularx}
\usepackage[table]{xcolor}
\usepackage{xcolor}
\usepackage{soul}

\usepackage{tikz}
\usetikzlibrary{arrows.meta}

\usepackage{subcaption}
\usepackage{float}

\title{The Transformer as a Polar State Estimator}

\author{%
  Peter Racioppo \\
  Independent Researcher \\
  Los Angeles, CA \\
  \texttt{pcracioppo@gmail.com} \\
}

\begin{document}

\maketitle

\begin{abstract}
We show that the core components of the Transformer---attention, residual
connections, normalization, and rotary encodings---arise from a single
state-estimation problem. We treat each token as a noisy reading of a rotating
direction on the hypersphere, taken at its own magnitude. Estimating this
direction unifies the architecture: rotational dynamics give rotary
encodings, robust estimation gives attention, normalization enforces the
spherical geometry, and the residual connection steps toward the estimate.
This reconciles normalized queries and keys with unnormalized values as
opposite assumptions on whether token magnitudes are independent or share a
common scale. Finally, the \emph{Polar Transformer} restores three elements
standard blocks drop: rotating values alongside keys, using key magnitudes as
confidence, and stepping toward the estimate rather than adding it.
\end{abstract}

\section{Introduction}

Attention, residual connections, normalization, and rotary encodings are
usually introduced as separate design choices, yet they appear together in
nearly every modern Transformer \citep{vaswani2023attentionneed}. We show
that these components are not independent heuristics, but arise together
from a single state-estimation problem. We model each token as reading a shared, rotating direction at its own
magnitude; separating the shared direction from token-specific lengths
induces polar coordinates and spherical likelihoods. This formulation unifies the architecture, reconciles normalized queries and keys with unnormalized values, and predicts concrete modifications to standard attention blocks.

Self-attention can be viewed as a batch state estimator: each token provides
a noisy observation of a shared latent state, and attention reweights all
observations for each query. When token magnitudes vary independently, the quantity they all share---and that attention can estimate---is the rotating direction. Given its magnitude, a reading's direction is von
Mises--Fisher, the sphere's analogue of a Gaussian, and the best estimate is
a precision-weighted average. In the resulting state estimator, normalization extracts the observed directions; the rotations of RoPE correspond to linear time-invariant (LTI) dynamics; making the average robust to outlying readings gives attention; and the residual connection steps toward the estimate. We call this estimator the \textit{Polar Filter}.

This model also reconciles normalizing queries and keys with leaving
values unnormalized. When each token carries its own magnitude, comparing
raw vectors conflates alignment with length---letting a long, poorly aligned
key outscore a short, well-aligned one. Isolating the shared direction
therefore requires normalizing queries and keys, with length entering only as a confidence weight. For values, by contrast, if tokens share a common
scale against fixed noise, longer vectors carry more signal; weighting
normalized directions by this confidence cancels the normalization and
reduces to a plain sum of unnormalized vectors. Standard attention is thus a hybrid of the two regimes, which we unify in a single model. The standard Transformer drops three pieces of the
Polar Filter: it discards key magnitudes as a confidence signal, it does not rotate values to the query's position, and it adds the attention output
directly to the state rather than stepping toward it---inflating norms at every layer \citep{xiong2020layer,sun2025curse}. The \textit{Polar Transformer} restores all three.

Our contributions are:
\begin{enumerate}
\item \textbf{Polar Model:} a rotating direction read by each token at its
own magnitude, under which each reading's direction is exactly von
Mises--Fisher.
\item \textbf{Polar Filter:} a closed-form robust estimator of the direction
whose steps are normalization, attention with rotary positional encodings, and the residual connection.
\item \textbf{An account of the Transformer:} QK normalization and unnormalized
values as opposite magnitude assumptions, and the residual as an update of the estimate.
\item \textbf{Polar Transformer:} a Transformer that keeps what the standard
block drops: values rotated to the current position, key magnitudes as
confidence, and a residual step toward the estimate.
\end{enumerate}

\section{Related Work}
\label{sec:RelatedWork}

\paragraph{Attention as estimation and filtering.}
Attention has been viewed as kernel smoothing
\citep{tsai2019transformerdissectionunifiedunderstanding}, as retrieval from
an associative memory \citep{ramsauer2021hopfieldnetworksneed}, and as
inference in latent-variable models
\citep{gabbur2021probabilistic,bianchessi2025bayesianattentionmechanismprobabilistic,bui2025revisitingkernelattentioncorrelated},
all static estimators. The Kalman filter \citep{originalkalmanfilter} can be approximated by Transformers
\citep{goel2024transformerrepresentkalmanfilter}, linear attention follows
from information-form filtering \citep{shaj2026kalmanlinearattentionparallel},
and state-space models share attention's form
\citep{gu2022efficientlymodelinglongsequences,dao2024transformersssmsgeneralizedmodels}.
Robust Filter Attention (RFA) \citep{racioppo2026rfa} derives attention as a
robust estimator of a Euclidean state: each token is a noisy observation,
transported to the query's time by the dynamics and weighted by its
propagated precision. We replace
the state with a direction read at token-specific magnitudes, and RFA is the
limit in which every magnitude is pooled. Kalman filters on manifolds \citep{hertzberg2011integrating} and Lie groups
\citep{barrau2017invariant} are recursive; the Polar
Filter uses all transported observations in parallel.

\paragraph{Hyperspherical readings of attention.}
Normalized representations have been studied as trajectories on the sphere
\citep{molina2023travelingwordsgeometricinterpretation} and modeled as
interacting particles that cluster with depth
\citep{geshkovski2024mathematicalperspectivetransformers}. With normalized
queries and keys, attention approximates the read operation of sparse
distributed memory \citep{bricken2021attention}, and pre-normalization
followed by attention performs clustering with von Mises--Fisher likelihoods
whose concentrations are set by the key norms \citep{schaeffer2024bridging}. Concurrently, SVFlow \citep{liao2026svflow} derives attention as a von Mises--Fisher
posterior and the residual connection with RMSNorm as a retraction, Hyper-SET
\citep{hu2026hyperset} derives symmetric attention, feedforward layers, and
RMSNorm from energy minimization on the hypersphere, and CDAT
\citep{zhang2026cdat} couples a mixture-of-vMF attention energy with
attractor dynamics. These works view attention as a static memory or an energy
descent, without a dynamical model of the latent state, and do not derive the
sum of unnormalized values. Transformer layers have also been related to
power iteration: with shared weights, successive layers amplify the principal
eigenvector of the value matrix \citep{li_powermethod}, continuous-time
attention dynamics converge to it \citep{altafini_multistability}, and looped
Transformers with layer normalization learn one power iteration per layer
\citep{wu2026loopedtransformerslayernormalization}. Here, power iteration
arises on the precision-weighted scatter matrix of the readings, and only
when the sign of each reading is unidentified (Appendix~\ref{sec:extensions}).

\paragraph{Normalized and geometric architectures.}
nGPT \citep{loshchilov2025ngptnormalizedtransformerrepresentation} keeps every
representation at unit norm, stepping it toward the attention output with
learned rates and renormalizing; in the Polar Filter, the step sizes follow
from each token's evidence and the prior, differ radially and tangentially,
and magnitudes are kept as confidence rather than normalized away. QUEST
\citep{govindarajan2026quest} normalizes only the keys, so that the query norm
sets each token's sharpness; in our model, the query's magnitude enters the
test as the level of shared-level keys. RiemannFormer
\citep{ji2025riemannformerframeworkattentioncurved} computes attention in
tangent spaces connected by parallel transport. XSA
\citep{zhai2026exclusiveselfattention} removes attention's component along the value (a tangent space projection), and GeoNorm \citep{zheng2026geonormunifyprenormpostnorm}
uses a geodesic step. In our model, both are forms of the estimator's directional step.

\paragraph{Positional encodings.}
RoPE \citep{su2023roformerenhancedtransformerrotary} rotates queries and keys,
ALiBi \citep{press2022trainshorttestlong} adds a distance bias, and extensions
add decay \citep[xPos;][]{sun2022lengthextrapolatabletransformer}, rotate
values \citep[RoPER;][]{harik2022roper}, or use Lie groups
\citep[LieRE;][]{ostmeier2025liere}. Here, RoPE's rotations are the dynamics,
value rotation follows from aggregating in a shared frame; adding decay and process
noise, as in RFA \citep{racioppo2026rfa}, would make a token's reliability
fall with lag.

\section{Methods}
\label{sec:Methods}

We model attention as an estimator of a latent direction on the hypersphere,
read by each token at its own magnitude, extending Robust Filter Attention
(RFA) \citep{racioppo2026rfa} from a Euclidean latent state. The directional
likelihood is von Mises--Fisher, whose maximizer is a precision-weighted
average in closed form, at any concentration. Sections~\ref{sec:rt_model}--\ref{sec:rt_filter}
derive the filter for a single stream of readings,
Section~\ref{sec:Transformer_as_RT_Filter} shows that the standard
Transformer treats key magnitudes as specific to each token and value
magnitudes as sharing one level, and Section~\ref{sec:polar_transformer} defines the
Polar Transformer.

\subsection{Dynamical Model}
\label{sec:rt_model}

A latent direction evolves under linear dynamics and is read by each token at
its own latent magnitude $m_j>0$, with Gaussian measurement noise,
\begin{equation}
\dot{\boldsymbol u}(t)=i\boldsymbol\Omega\,\boldsymbol u(t),
\qquad
\boldsymbol z_j=\boldsymbol C\big(m_j\,\boldsymbol u(t_j)+\boldsymbol\varepsilon_j\big),
\qquad
\boldsymbol\varepsilon_j\sim\mathcal N\big(\boldsymbol 0,\sigma_j^2\boldsymbol I\big),
\label{eq:dynamical_model}
\end{equation}
where $\boldsymbol z_j\in\mathbb R^d$ is the token embedding, the direction
$\boldsymbol u(t)\in\mathbb C^d$ is in the eigenbasis of the dynamics, with
$\boldsymbol\Omega=\mathrm{diag}(\omega_k)$ real and its frequencies in
conjugate pairs, and $\boldsymbol C$ maps this basis to the embedding space.
We identify $\mathbb C^d$ with $\mathbb R^{2d}$, with inner product
$\langle\boldsymbol a,\boldsymbol b\rangle:=\mathrm{Re}(\boldsymbol a^\dagger\boldsymbol b)$.
The flow $\boldsymbol\Phi(\tau)=e^{i\boldsymbol\Omega\tau}$ is unitary, so the
direction stays on the unit sphere $\mathcal S^{2d-1}$ and rotates,
$\boldsymbol u(t)=\boldsymbol\Phi(t)\boldsymbol u(0)$; we write
$\boldsymbol u_j:=\boldsymbol u(t_j)$. The latent magnitudes scatter about a
shared level, $m_j\sim\mathcal N(\ell,\zeta^2)$, and the noise variance
$\sigma_j^2=\eta^2+\rho^2m_j^2$ has a part of fixed size and a part
that scales with the magnitude. The direction is shared by all tokens and
carries the dynamics, while each magnitude is specific to its token, up to
the shared level. The magnitude model is what requires polar coordinates: since each belongs to its token, the quantity that all tokens share, and that the
filter can estimate, is the direction. For simplicity, we consider the case of deterministic, purely rotational LTI dynamics, and leave the stochastic and time-varying cases for future work.

\paragraph{Directional likelihood.}
Each token yields an estimate of its state,
$\hat{\boldsymbol x}_j=\boldsymbol C^{-1}\boldsymbol z_j=\hat m_j\hat{\boldsymbol u}_j$,
with observed magnitude $\hat m_j=\|\hat{\boldsymbol x}_j\|$ and observed
direction $\hat{\boldsymbol u}_j\in\mathcal S^{2d-1}$. To find the density of
the direction given the magnitude, we write the Gaussian density of
$\hat{\boldsymbol x}_j$ in terms of the two. Its exponent is
$-\|\hat{\boldsymbol x}_j-m_j\boldsymbol u_j\|^2/2\sigma_j^2$, and since
$\|\boldsymbol u_j\|=\|\hat{\boldsymbol u}_j\|=1$,
\[
\|\hat{\boldsymbol x}_j-m_j\boldsymbol u_j\|^2
=\hat m_j^2+m_j^2-2\hat m_jm_j\langle\boldsymbol u_j,\hat{\boldsymbol u}_j\rangle .
\]
Only the last term depends on the direction, and the polar volume factor
$\hat m_j^{2d-1}$ depends on the magnitude alone; conditioning on $\hat m_j$
absorbs all remaining factors into the normalizing constant, leaving the
von Mises--Fisher (vMF) density \citep{mardia2000directional}
\begin{equation}
f(\hat{\boldsymbol u}_j;\boldsymbol u_j,\kappa_j)
=C_{2d}(\kappa_j)\exp\big(\kappa_j\langle\boldsymbol u_j,\hat{\boldsymbol u}_j\rangle\big),
\qquad
\kappa_j=\frac{\hat m_j\,m_j}{\sigma_j^2},
\label{eq:vmf}
\end{equation}
the analogue on the sphere of an isotropic Gaussian: uniform at $\kappa_j=0$,
and concentrated about $\boldsymbol u_j$ as $\kappa_j$ grows. By rotational
symmetry, the distribution of the observed magnitude does not depend on the
direction, so conditioning on it loses no information about the direction.

\paragraph{Estimating the latent magnitude.}
The concentration depends on the latent magnitude, which is never observed.
Replacing it by its posterior mean, a precision-weighted average of the
reading's length and the level, gives, for a reading that agrees with the
direction, the concentration
\begin{equation}
\kappa_j=\frac{\hat m_j\,m_j^*}{\sigma_j^2},
\qquad
m_j^*=\lambda_j\hat m_j+(1-\lambda_j)\,\ell,
\qquad
\lambda_j=\frac{\zeta^2}{\sigma_j^2+\zeta^2},
\label{eq:pooled_concentration}
\end{equation}
where $m_j^*$ is the posterior mean of the latent magnitude, and $\lambda_j$
weighs the scatter of the magnitudes against the reading's noise. The
concentration has two ends. With \emph{token-specific magnitudes},
$\lambda_j\to1$, the scatter dominates the noise, each reading reveals its
own token's magnitude, and
\begin{equation}
\kappa_j=\frac{\hat m_j^2}{\eta^2+\rho^2\hat m_j^2},
\label{eq:concentration}
\end{equation}
so that magnitude acts as confidence. Noise of fixed variance ($\eta$)
perturbs larger tokens less, so their directions are more concentrated; noise
proportional to the magnitude ($\rho$) does not, and $\eta\to0$ gives a flat
precision. With \emph{shared-level magnitudes}, $\lambda_j\to0$, the noise
dominates the scatter, the magnitudes share one level, and, for a common
noise $\sigma$,
\[
\kappa_j=\frac{\hat m_j\,\ell}{\sigma^2},
\qquad
\sum_j\kappa_j\hat{\boldsymbol u}_j=\frac{\ell}{\sigma^2}\sum_j\hat{\boldsymbol x}_j,
\]
a concentration linear in the observed magnitude, so that weighting the
directions by it sums the unnormalized readings. Figure~\ref{fig:rt_dynamics}
illustrates the model.

\paragraph{Transported observations.}
We transport each observed direction to the current time,
$\hat{\boldsymbol u}_{ij}:=\boldsymbol\Phi_{ij}\hat{\boldsymbol u}_j$, where
$\boldsymbol\Phi_{ij}:=\boldsymbol\Phi(t_i-t_j)$. Since $\boldsymbol\Phi_{ij}$ is unitary, it preserves areas on the sphere and
inner products:
\[
\langle\boldsymbol u_i,\hat{\boldsymbol u}_{ij}\rangle
=\langle\boldsymbol\Phi_{ij}\boldsymbol u_j,\boldsymbol\Phi_{ij}\hat{\boldsymbol u}_j\rangle
=\langle\boldsymbol u_j,\hat{\boldsymbol u}_j\rangle .
\]
Substituting this into Eq.~\ref{eq:vmf} shows that $\hat{\boldsymbol u}_{ij}$
is vMF about $\boldsymbol u_i$ with concentration $\kappa_j$. Thus, $\hat{\boldsymbol u}_{ij}$ is vMF about the current direction
$\boldsymbol u_i$, with the concentration of token $j$: every transported
direction reads the current direction, at every lag. Below, we write $\kappa_{ij}$
to allow for lag dependence; under deterministic dynamics, it equals
$\kappa_j$. The transported readings are
$\hat{\boldsymbol x}_{ij}:=\boldsymbol\Phi_{ij}\hat{\boldsymbol x}_j=\hat m_j\hat{\boldsymbol u}_{ij}$.

\begin{figure}[!b]
\centering
\begin{subfigure}[b]{0.35\linewidth}
\centering
\includegraphics[
    width=\linewidth,
    trim={5.0cm 7.0cm 6.0cm 3.0cm},
    clip
]{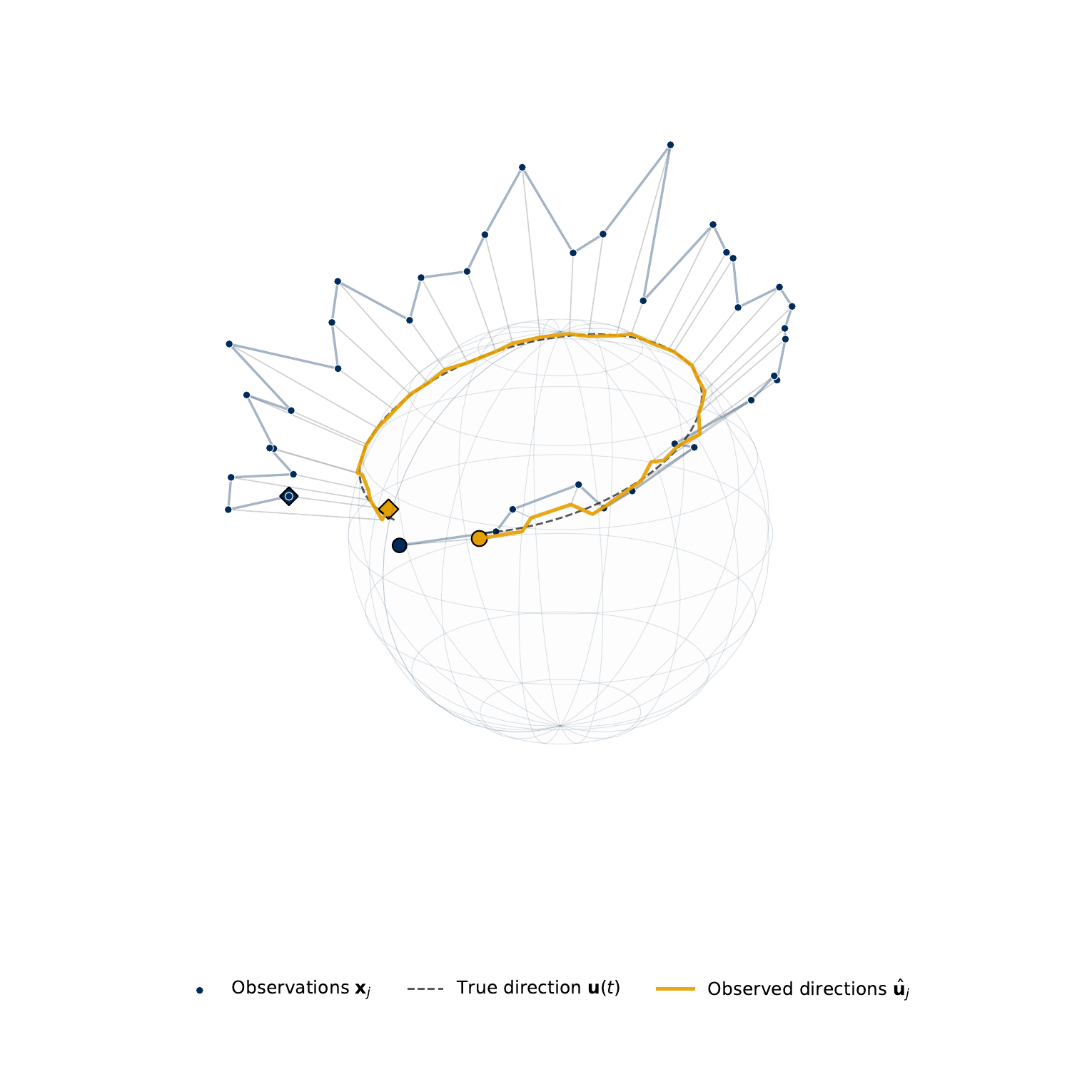}
\caption{}
\label{fig:rt_dynamics_obs}
\end{subfigure}
\hfill
\begin{subfigure}[b]{0.56\linewidth}
\centering
\includegraphics[
    width=\linewidth,
    trim={1.3cm 0.8cm 0.8cm 0.0cm},
    clip
]{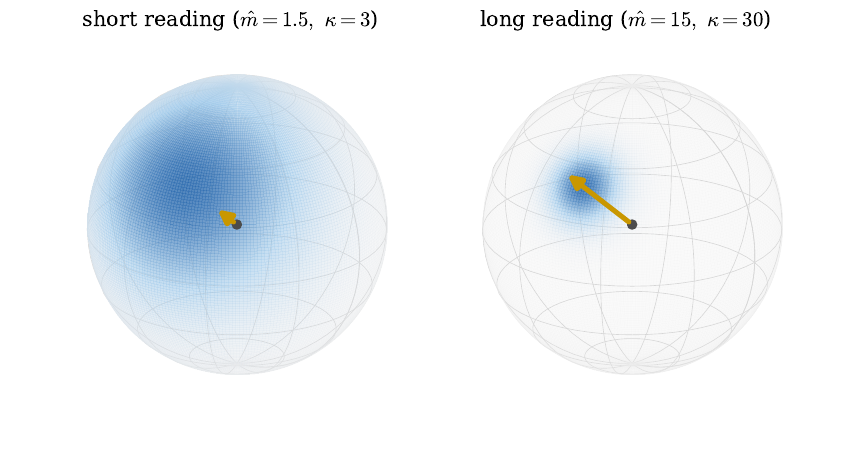}
\caption{}
\label{fig:rt_dynamics_vmf}
\end{subfigure}
\caption{(a) Each observation $\hat{\boldsymbol x}_j$ (blue) combines a
reading of the shared direction on the unit sphere $\boldsymbol u(t_j)$
(gold) with the token's magnitude $\hat m_j$ (gray); the actual direction is
the black dotted line. (b) The von Mises--Fisher density of a reading's direction for a short and a long observation about the same direction $\boldsymbol u$ (gold). The longer
the observation, the higher its concentration $\kappa$, and the more tightly
its density concentrates around $\boldsymbol u$.}
\label{fig:rt_dynamics}
\end{figure}

\subsection{Directional Likelihood}
\label{sec:likelihood_decomposition}

The
negative log-likelihood of the current direction is
\begin{equation}
\mathcal L(\boldsymbol u)
=-\sum_{j\le i}\kappa_{ij}\langle\boldsymbol u,\hat{\boldsymbol u}_{ij}\rangle+\text{const}
=\tfrac12\sum_{j\le i}\kappa_{ij}\|\boldsymbol u-\hat{\boldsymbol u}_{ij}\|^2+\text{const},
\label{eq:rt_nll}
\end{equation}
where the second form uses
$\|\boldsymbol u-\hat{\boldsymbol u}\|^2=2\big(1-\langle\boldsymbol u,\hat{\boldsymbol u}\rangle\big)$
for unit vectors. The objective is therefore an exact weighted least-squares
problem in the chordal distance on the sphere, with each reading weighted by
its precision. Minimizing Eq.~\ref{eq:rt_nll} over unit vectors is equivalent to maximizing
$\langle\boldsymbol u,\bar{\boldsymbol u}_i\rangle$, where $\bar{\boldsymbol u}_i$
is the precision-weighted average of the transported directions. By
Cauchy--Schwarz, this is uniquely maximized when $\boldsymbol u$ aligns with
$\bar{\boldsymbol u}_i$, giving
\begin{equation}
\boldsymbol u_i^*=\frac{\bar{\boldsymbol u}_i}{\|\bar{\boldsymbol u}_i\|},
\qquad
\bar{\boldsymbol u}_i=\sum_{j\le i}\frac{\kappa_{ij}}{\sum_{j'\le i}\kappa_{ij'}}\,\hat{\boldsymbol u}_{ij}.
\label{eq:consensus}
\end{equation}
For shared-level magnitudes, whose precision is proportional to the observed magnitude, the
estimate of the direction is the direction of the sum of the unnormalized
readings, $\bar{\boldsymbol u}_i=\sum_j\hat{\boldsymbol x}_{ij}/\sum_j\hat m_j$.

\subsection{Robust Estimation}
\label{sec:robust}

A token may be reliable for one query and not another, and under the
quadratic loss of Eq.~\ref{eq:rt_nll}, a single inconsistent reading can pull
the estimate arbitrarily far. We therefore replace the quadratic loss by a
robust penalty of each reading's squared distance from the candidate
direction,
\begin{equation}
\mathcal L_\rho(\boldsymbol u)=\sum_{j\le i}\rho\big(d_{ij}^2(\boldsymbol u)\big),
\qquad
d_{ij}^2(\boldsymbol u)
=\kappa_{ij}\big\|\boldsymbol u-\hat{\boldsymbol u}_{ij}\big\|^2
=2\kappa_{ij}\big(1-\langle\boldsymbol u,\hat{\boldsymbol u}_{ij}\rangle\big),
\label{eq:mahalanobis}
\end{equation}
so that the estimator is an M-estimator \citep{huber1964}. The distance is
the directional part of the likelihood, measured from perfect agreement:
since the magnitude carries no information about the direction, a reading is
tested by its direction alone, at a scale set by its magnitude, yielding a
single scalar weight $w_{ij}$ for the token. Its stationary condition is that of
Eq.~\ref{eq:rt_nll} with each $\kappa_{ij}$ replaced by
$w_{ij}\kappa_{ij}$, where $w_{ij}=\rho'\big(d_{ij}^2\big)$. Evaluating the
weights at the current estimate and solving the weighted least squares problem is one step of iteratively reweighted least squares (IRLS). To reject outliers while retaining a smooth, heavy-tailed decay, we
use the bounded penalty $\rho(s)=d\big[1-(1+s/\nu)^{-\nu/d}\big]$, with
$\nu=\nu_td$. Its weight, the Student-$t$ kernel,
\begin{equation}
w_{ij}=\Big(1+\frac{d_{ij}^2}{\nu_td}\Big)^{-(\nu_t+1)},
\label{eq:student_t_weight}
\end{equation}
falls polynomially with distance and recovers Welsch's exponential kernel
$w_{ij} = \exp(-d_{ij}^2/d)$---and hence standard softmax attention---as
$\nu_t\to\infty$; either kernel may be used. Both are bounded penalties,
so the influence of a reading that disagrees with the estimate falls to zero.
The division by $d$ normalizes the logits by the feature dimension, as the
$1/\sqrt d$ scaling does in standard attention.

\subsection{The Polar Filter}
\label{sec:rt_filter}

The Polar Filter takes one robust, proximal step on Eq.~\ref{eq:mahalanobis},
starting at the current state $\hat{\boldsymbol x}_i=\hat m_i\hat{\boldsymbol u}_i$;
the readings are tested against the current direction, attention computes the
robust estimate, and the residual connection takes a step toward it.

\paragraph{Weights and estimate.}
The weights are evaluated at the current state. Letting
$c_{ij}=\langle\hat{\boldsymbol u}_i,\hat{\boldsymbol u}_{ij}\rangle$, each
reading has distance $d_{ij}^2=2\kappa_{ij}(1-c_{ij})$ and weight
$w_{ij}=\rho'(d_{ij}^2)$, and their total, the evidence
$W_i=\sum_{j\le i}w_{ij}$, counts the readings that agree with the current
direction. The robust consensus weights each transported direction by its
robust weight and its precision,
\begin{equation}
\bar{\boldsymbol u}_i=\sum_{j\le i}A_{ij}\hat{\boldsymbol u}_{ij},
\qquad
A_{ij}=\frac{w_{ij}\kappa_{ij}}{\sum_{j'\le i}w_{ij'}\kappa_{ij'}}.
\label{eq:directional_update}
\end{equation}
With shared-level magnitudes and a common noise, $\kappa_{ij}\propto\hat m_j$,
so dividing the numerator and denominator of $A_{ij}$ by $W_i$ gives
$A_{ij}=\pi_{ij}\hat m_j/\bar m_i$, with normalized weights
$\pi_{ij}=w_{ij}/W_i$ and robust mean magnitude
$\bar m_i=\sum_j\pi_{ij}\hat m_j$. The consensus is then the direction of the
weighted sum of the transported readings,
\[
\bar{\boldsymbol x}_i=\sum_{j\le i}\pi_{ij}\hat{\boldsymbol x}_{ij}=\bar m_i\,\bar{\boldsymbol u}_i.
\]
The robust estimate of the level---the minimizer of
$\sum_{j\le i}w_{ij}\|\hat{\boldsymbol x}_{ij}-\ell\hat{\boldsymbol u}_i\|^2$
over $\ell$---is its radial component,
$\hat\ell_i=\langle\hat{\boldsymbol u}_i,\bar{\boldsymbol x}_i\rangle=\sum_j\pi_{ij}\hat m_jc_{ij}$.

\paragraph{Proximal step.}
The current state is not a raw reading: it is the estimate carried from
earlier layers, so we treat it as a Gaussian prior. Because magnitude and
direction carry different information, the prior's uncertainty need not be
isotropic; the only covariance that singles out no direction other than
$\hat{\boldsymbol u}_i$ is radial--tangential about it. We give the prior
radial and tangential precisions $c_r$ and $c_t$ (measured in units of
agreeing readings), so that each component of the update is weighted by its
share of the evidence, $\beta_{r,i}=W_i/(c_r+W_i)$ and
$\beta_{t,i}=W_i/(c_t+W_i)$.

With shared-level magnitudes and common noise, the posterior mean is exact.
Holding the weights fixed, the readings act as a single isotropic observation
at their weighted mean $\bar{\boldsymbol x}_i$ with precision $W_i$. Because
the prior is diagonal in the radial--tangential frame, the posterior
decouples into two independent steps from $\hat{\boldsymbol x}_i$:
\begin{itemize}
\item \emph{Tangentially} (orthogonal to $\hat{\boldsymbol u}_i$), the state
moves toward $\bar{\boldsymbol x}_i$ by fraction $\beta_{t,i}$, via the tangent-space projection
$\boldsymbol P_T(\hat{\boldsymbol u}_i)(\bar{\boldsymbol x}_i-\hat{\boldsymbol x}_i)=\boldsymbol P_T(\hat{\boldsymbol u}_i)\bar{\boldsymbol x}_i$,
where $\boldsymbol P_T(\boldsymbol u)\boldsymbol v=\boldsymbol v-\boldsymbol u\langle\boldsymbol u,\boldsymbol v\rangle$
removes the component along $\boldsymbol u$, so $\boldsymbol P_T(\hat{\boldsymbol u}_i)\hat{\boldsymbol x}_i=\boldsymbol 0$ (see
Fig.~\ref{fig:transformer_manifold_geometry} in Appendix~\ref{sec:Implementation}).
\item \emph{Radially} (along $\hat{\boldsymbol u}_i$), the magnitude $\hat m_i$
moves toward $\hat\ell_i=\langle\hat{\boldsymbol u}_i,\bar{\boldsymbol x}_i\rangle$
by fraction $\beta_{r,i}$.
\end{itemize}
Combining the two gives the residual update:
\begin{equation}
\hat{\boldsymbol x}_i^{+}
=\hat{\boldsymbol x}_i
+\beta_{t,i}\,\boldsymbol P_T(\hat{\boldsymbol u}_i)\,\bar{\boldsymbol x}_i
+\beta_{r,i}\,(\hat\ell_i-\hat m_i)\,\hat{\boldsymbol u}_i.
\label{eq:pooled_step}
\end{equation}

Figure~\ref{fig:spherical_filtering_mechanics} illustrates the update. With token-specific magnitudes, $\lambda_i>0$, the token's own reading also
informs its magnitude, which shrinks the radial step by a factor
$1-\lambda_i$; at $\lambda_i=1$, the magnitude is fixed and only the direction
turns. Appendix~\ref{sec:extensions} treats the general case with a single objective that serves as both the robust test and the estimation objective. The update keeps the same form, with each reading weighted by its own concentration, and it reduces to Eq.~\ref{eq:pooled_step} at $\lambda_i=0$.

\begin{figure}[b]
     \centering
     \begin{subfigure}[t]{0.4\textwidth}
         \centering
         \includegraphics[
        width=\linewidth,
        trim={2.0cm 0.54cm 2cm 2.7cm},
        clip
    ]{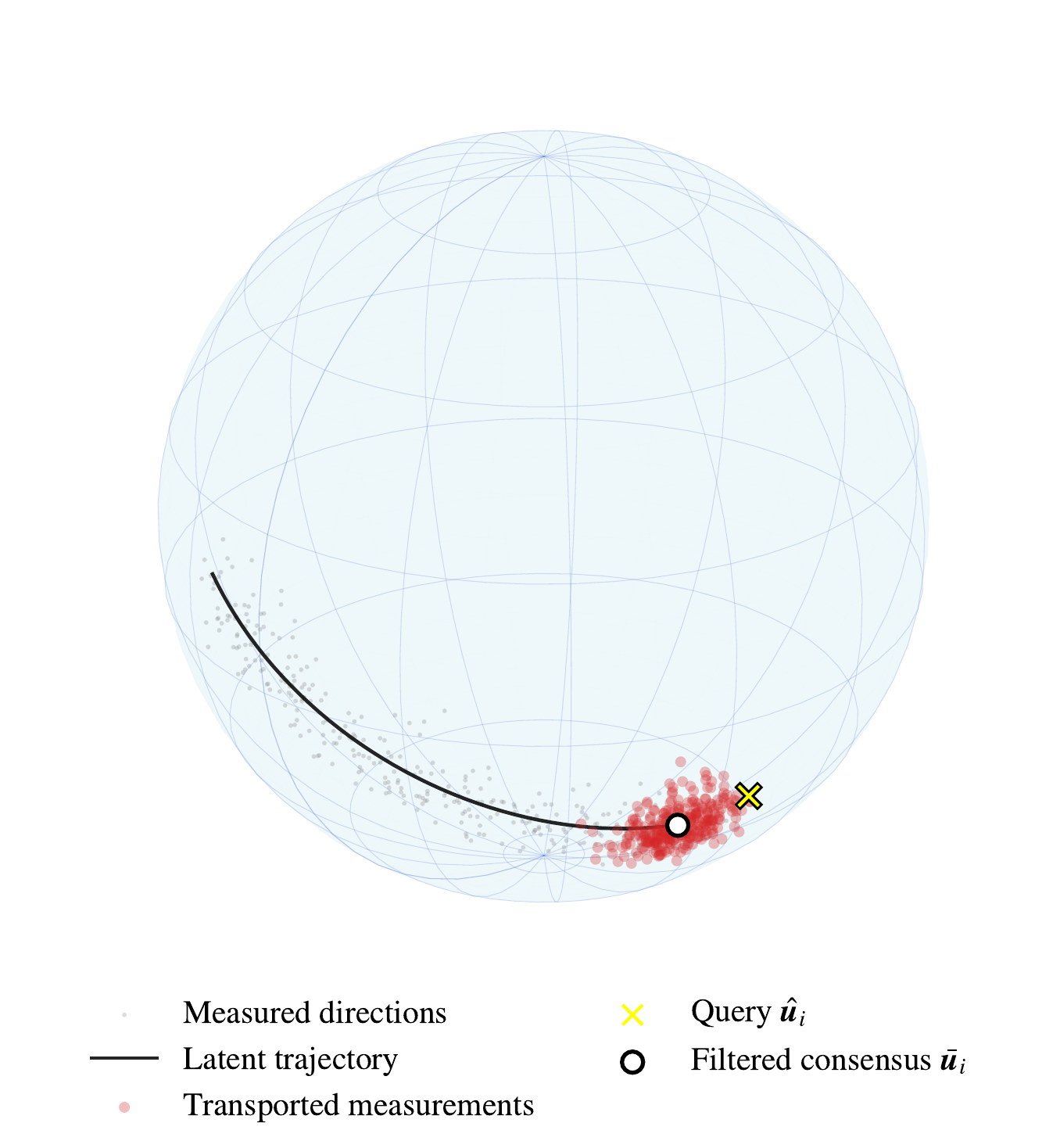}
         \label{fig:pulled_forward}
     \end{subfigure}
     \hfill
     \begin{subfigure}[t]{0.4\textwidth}
         \centering
         \includegraphics[
        width=\linewidth,
        trim={2.0cm 0.7cm 2cm 1.0cm},
        clip
    ]{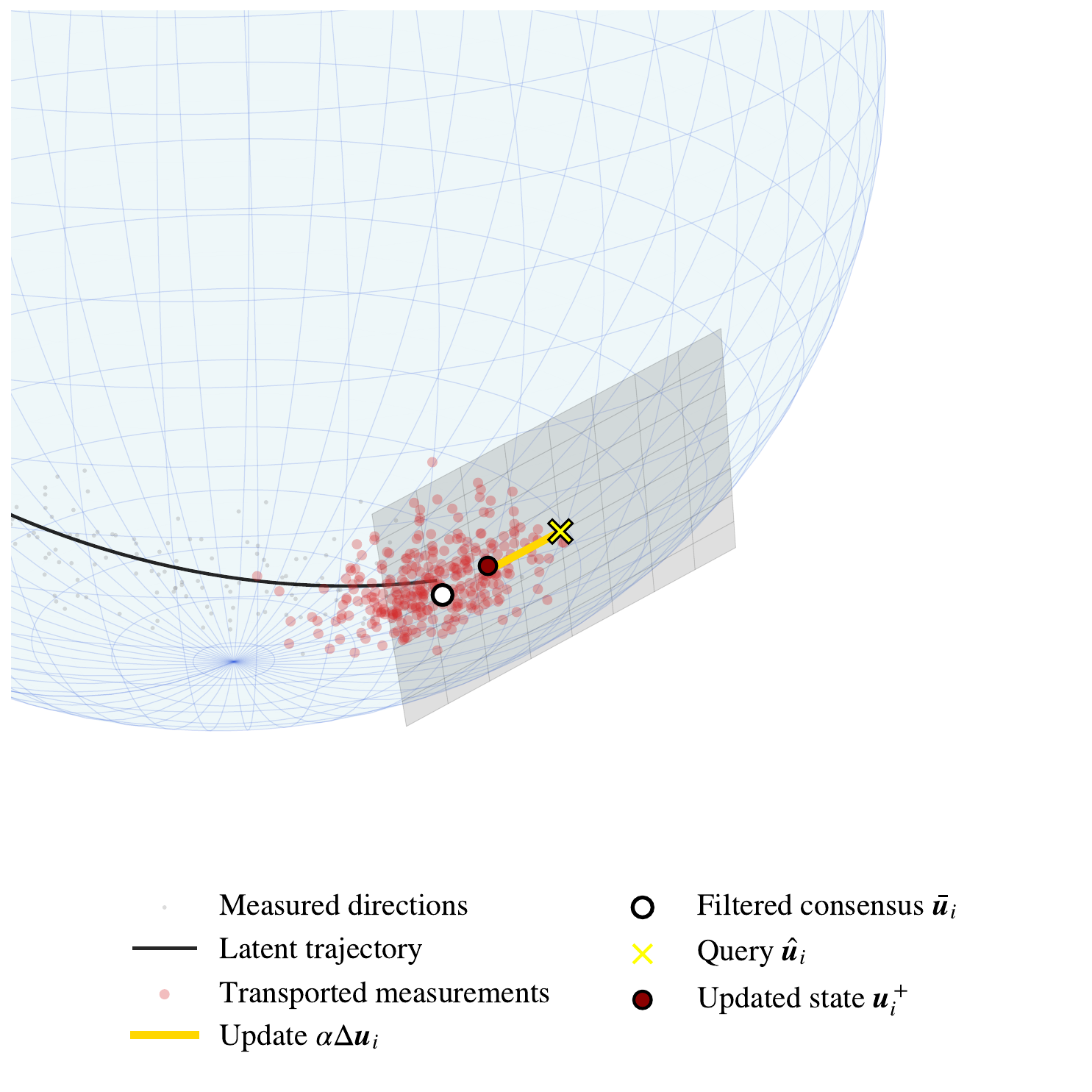}
         \label{fig:tangent_step}
     \end{subfigure}

     \vspace{0.5em}
    \caption{The Polar Filter. Left: each observation is rotated to the query's
    time. Right: their robust weighted average gives a step for the direction
    along the sphere; a separate step (not pictured) moves the magnitude toward the shared level.}
     \label{fig:spherical_filtering_mechanics}
\end{figure}

\subsection{The Transformer as an Approximate Polar Filter}
\label{sec:Transformer_as_RT_Filter}

We now instantiate the Polar Filter as a single-head Transformer; the multi-head case is discussed in Appendix~\ref{sec:multihead_attention}. The
single stream of readings is relaxed into separate queries, keys, and values,
$\boldsymbol Q=\boldsymbol W_q\boldsymbol Z$,
$\boldsymbol K=\boldsymbol W_k\boldsymbol Z$, and
$\boldsymbol V=\boldsymbol W_v\boldsymbol Z$, with columns
$\boldsymbol q_i$, $\boldsymbol k_j$, and $\boldsymbol v_j$, computed from the
residual stream $\boldsymbol Z$. The query plays the role of the current
state, the keys are the readings tested against it, and the values are the
readings aggregated into the estimate. The input projections absorb
$\boldsymbol C^{-1}$, and $\boldsymbol W_o$ absorbs $\boldsymbol C$. Each
kind of reading has its own magnitudes, $m_{q,i}$, $m_{k,j}$, and $m_{v,j}$,
and directions $\hat{\boldsymbol u}_{q,i}$, $\hat{\boldsymbol u}_{k,j}$, and
$\hat{\boldsymbol u}_{v,j}$, and may be in either of the two regimes of 
Section~\ref{sec:rt_model}. RMSNorm \citep{zhang2019rms} with a scalar gain realizes the
unit-sphere constraint at a learned radius $r$, and magnitudes are measured
before normalization.

\paragraph{Transport and attention.}
Writing the transport as elementwise rotations
$\tilde{\boldsymbol\Phi}^{\pm}[k,i]=e^{\pm i\omega_k t_i}$, so that
$\boldsymbol\Phi_{ij}=\tilde{\boldsymbol\Phi}^{+}_i\tilde{\boldsymbol\Phi}^{-}_j$,
each vector is rotated once into a common frame,
$\tilde{\boldsymbol u}=\tilde{\boldsymbol\Phi}^{-}\odot\hat{\boldsymbol u}$.
Since the rotations are unitary,
$c_{ij}=\langle\hat{\boldsymbol u}_{q,i},\boldsymbol\Phi_{ij}\hat{\boldsymbol u}_{k,j}\rangle
=\mathrm{Re}(\tilde{\boldsymbol u}_{q,i}^\dagger\tilde{\boldsymbol u}_{k,j})$,
which is algebraically identical to RoPE. With token-specific key
magnitudes, the precision of a key is that of Eq.~\ref{eq:concentration},
$\kappa^k_j=m_{k,j}^2/(\eta_k^2+\rho_k^2m_{k,j}^2)$, and under the
exponential kernel, the attention weights are
\[
\pi_{ij}
=
\operatorname{softmax}_{j,\,\mathrm{causal}}
\!\Big(
-\tfrac{2}{d}\,\kappa^k_j\big(1-\mathrm{Re}(\tilde{\boldsymbol u}_{q,i}^\dagger\tilde{\boldsymbol u}_{k,j})\big)
\Big).
\]
With noise proportional to the key magnitude, $\eta_k\to0$, every key is
equally confident, $\kappa^k_j=1/\rho_k^2$, and the weights are those of
attention with QK normalization \citep{henry2020querykeynormalization} and RoPE, at inverse temperature
$2/(d\rho_k^2)$. Without QK normalization, the logit
$\boldsymbol q_i^\top\boldsymbol k_j$ is, up to scale and a term in the query
alone, the shared-level key test $-\|\boldsymbol q_i-\boldsymbol k_j\|^2/2$,
with the query's magnitude as the level, but without its term
$-\|\boldsymbol k_j\|^2/2$, which penalizes long keys; without it, a longer
key can outscore a better-aligned one. With shared-level value magnitudes and a common value
noise, the estimate is the weighted sum of the transported values,
aggregated in the common frame and rotated back,
\[
\bar{\boldsymbol v}_i=\sum_{j\le i}\pi_{ij}\boldsymbol\Phi_{ij}\boldsymbol v_j
=\tilde{\boldsymbol\Phi}^{+}_i\odot\sum_{j\le i}\pi_{ij}\tilde{\boldsymbol v}_j,
\]
from the unnormalized values. The standard Transformer computes
$\sum_j\pi_{ij}\boldsymbol v_j$: the same estimate, exact given its weights,
without value transport.

\paragraph{The residual update.}
Since $\boldsymbol z_i=\boldsymbol C\hat{\boldsymbol x}_i$ and the values
have shared-level magnitudes, the step is that of Eq.~\ref{eq:pooled_step},
computed in value coordinates with $\boldsymbol W_o$ in place of
$\boldsymbol C$. Letting $\hat \ell_i=\langle\hat{\boldsymbol u}_{v,i},\bar{\boldsymbol v}_i\rangle$,
\begin{equation}
\boldsymbol z_i^{+}
=\boldsymbol z_i+\boldsymbol W_o\,\Delta\boldsymbol v_i,
\qquad
\Delta\boldsymbol v_i
=\beta_{t,i}\,\boldsymbol P_T(\hat{\boldsymbol u}_{v,i})\,\bar{\boldsymbol v}_i
+\beta_{r,i}\,(\hat \ell_i-m_{v,i})\,\hat{\boldsymbol u}_{v,i}.
\label{eq:rt_transformer_update}
\end{equation}
\begin{itemize}
\item With $c_r\to\infty$ and $c_t\to0$, the step is
$\boldsymbol P_T(\hat{\boldsymbol u}_{v,i})\sum_j\pi_{ij}\boldsymbol\Phi_{ij}\boldsymbol v_j$.
This is Exclusive Self-Attention (XSA) \citep{zhai2026exclusiveselfattention}, which
removes from the attention output its component along the token's own value,
here with value transport. XSA is the step of a token whose value magnitude
is held by its prior, since it is the estimate carried from earlier layers.
\item With $c_r=c_t$, the step is $\beta_i(\bar{\boldsymbol v}_i-\boldsymbol v_i)$,
a step toward the estimate in the ambient space.
\item The standard Transformer adds $\boldsymbol W_o\bar{\boldsymbol v}_i$, which
is not a step of the filter: its tangential part is that of the full step,
but its radial part adds the level $\hat \ell_i$ to the magnitude at every layer
rather than moving the magnitude toward it
---causing residual-stream norms to grow with depth \citep{xiong2020layer,sun2025curse}
(Appendix~\ref{sec:architectural_implications}).
\item The Transformer's feedforward layer (FFN) lies outside the present framework (see Appendix~\ref{sec:architectural_implications} for additional discussion).
\end{itemize}
The standard Transformer is thus a hybrid of the two ends of
Section~\ref{sec:rt_model}: its keys take token-specific magnitudes, with
noise proportional to them, so QK normalization is their test, and its values
take a shared level, so the sum of unnormalized values is their estimate. It
then omits value transport and adds its estimate whole. One possible reading of the split is that keys only feed the test, so a key's magnitude need only say how confident it
is, while values carry the state, so a shared level lets the context set
each token's magnitude.

\subsection{The Polar Transformer}
\label{sec:polar_transformer}

The Polar Transformer keeps the Transformer's pairing of token-specific keys
and shared-level values and implements Eq.~\ref{eq:rt_transformer_update}.
We do not claim that this pairing is optimal; the model makes each choice
explicit. The Polar
Transformer makes three changes to the standard block. First, it rotates
the values to the current position and counter-rotates the estimate, as in RFA. Second, it takes the proximal step, with
$c_r$ and $c_t$ learned per head, so that each head can lie anywhere between
XSA and the full step toward the estimate. Third, it puts the key magnitudes
in the test, $\kappa^k_j=m_{k,j}^2/(\eta_k^2+\rho_k^2m_{k,j}^2)$: larger
keys are tested more strictly when $\eta_k>0$, and the test reduces to QK
normalization at $\eta_k=0$. A long key then counts as a confident reading:
it is selected decisively when it agrees with the query and rejected
decisively when it does not, while a short key is weighed loosely either way. It uses one softmax and the Transformer's aggregation of values,
with five scalars per head: $\eta_k$, $\rho_k$, $c_r$, $c_t$, and $\nu_t$. Under the exponential kernel ($\nu_t\to\infty$), the logits split into a
per-key bias and a rescaled dot product, so standard fused attention applies
directly; finite $\nu_t$ requires only an elementwise logit map in a single
tiled kernel (Appendix~\ref{sec:Implementation}). Figure~\ref{fig:model_comparison} compares the two layers, and
Algorithm~\ref{alg:polar_transformer} 
(Appendix~\ref{sec:Implementation}) gives the full layer. As an extension,
the Learned-Scatter Polar Transformer (Appendix~\ref{sec:extensions}) learns
the scatter of the key and value magnitudes per head, so that each kind of reading may sit anywhere between
the two ends, and with it how far the context may change a token's
magnitude; whether this helps is an empirical question, which we defer to
later work.

\begin{figure}[!t]
\centering
\resizebox{0.56\textwidth}{!}{%
\begin{tikzpicture}[
  font=\small, >={Stealth[length=1.8mm]},
  blk/.style={draw, rounded corners=1.5pt, align=center, minimum width=3.6cm,
              minimum height=0.95cm, inner sep=2pt, line width=0.55pt},
  prj/.style={draw=gray!50, rounded corners=1.5pt, align=center, minimum width=3.6cm,
              minimum height=0.55cm, inner sep=1pt, line width=0.45pt,
              font=\scriptsize, text=black!70},
  std/.style={blk, fill=gray!8, draw=gray!55},
  pol/.style={blk, fill=orange!18, draw=orange!75!black},
  gen/.style={blk, fill=violet!12, draw=violet!65!black},
  arr/.style={->, line width=0.6pt, draw=black!70},
  side/.style={->, line width=0.6pt, draw=orange!75!black},
  ttl/.style={font=\normalsize\bfseries},
  hyb/.style={font=\scriptsize, text=gray!60!black, align=right, anchor=east},
]
\def\ya{0}\def\yb{1.2}\def\yc{2.55}\def\yd{3.95}\def\ye{5.35}\def\yf{6.75}\def\yg{7.95}
\begin{scope}[xshift=0cm]
\node[ttl] at (0,8.8) {Transformer (QK norm)};
\node[prj] (a1) at (0,\ya) {Input projections $\boldsymbol W_q,\boldsymbol W_k,\boldsymbol W_v$};
\node[std] (a2) at (0,\yb) {Normalize, rotate $\boldsymbol Q,\boldsymbol K$};
\node[std] (a3) at (0,\yc) {Cosine\\$\mathrm{Re}\big(\tilde{\boldsymbol u}_{q,i}^\dagger\tilde{\boldsymbol u}_{k,j}\big)$};
\node[std] (a4) at (0,\yd) {Softmax};
\node[std] (a5) at (0,\ye) {Sum of unnormalized\\values};
\node[std] (a6) at (0,\yf) {Add the estimate $\bar{\boldsymbol v}$};
\node[prj] (a7) at (0,\yg) {Output projection $\boldsymbol W_o$};
\foreach \p/\q in {a1/a2,a2/a3,a3/a4,a4/a5,a5/a6,a6/a7} \draw[arr] (\p) -- (\q);
\draw[arr] (0,-0.75) -- (a1); \draw[arr] (a7) -- (0,8.45);
\node[hyb] at (-1.95,\yc) {keys:\\token-specific};
\node[hyb] at (-1.95,\ye) {values:\\shared level};
\end{scope}
\begin{scope}[xshift=4.9cm]
\node[ttl] at (0,8.8) {Polar Transformer};
\node[prj] (b1) at (0,\ya) {Input projections $\boldsymbol W_q,\boldsymbol W_k,\boldsymbol W_v$};
\node[pol] (b2) at (0,\yb) {Normalize, rotate $\boldsymbol Q,\boldsymbol K$;\\rotate $\boldsymbol V$};
\node[pol] (b3) at (0,\yc) {Robust test\\$2\kappa^k_j(1-c_{ij})$, $\kappa^k_j$ from $m_{k,j}$};
\node[pol] (b4) at (0,\yd) {Softmax,\\evidence $W_i$};
\node[pol] (b5) at (0,\ye) {Sum of unnormalized\\values, counter-rotate};
\node[pol] (b6) at (0,\yf) {Proximal step\\$ \beta_t\boldsymbol P_T\bar{\boldsymbol v}+\beta_r(\hat \ell-m_v)\hat{\boldsymbol u}_v$};
\node[prj] (b7) at (0,\yg) {Output projection $\boldsymbol W_o$};
\foreach \p/\q in {b1/b2,b2/b3,b3/b4,b4/b5,b5/b6,b6/b7} \draw[arr] (\p) -- (\q);
\draw[arr] (0,-0.75) -- (b1); \draw[arr] (b7) -- (0,8.45);
\draw[side] (b1.east) -- ++(0.35,0) |- node[pos=0.22, right, font=\scriptsize,
  text=orange!60!black, align=left] {$\boldsymbol v_i$\\current\\state} (b6.east);
\end{scope}
\begin{scope}[shift={(0.9,-1.35)}, font=\scriptsize]
\node[std, minimum width=0.45cm, minimum height=0.3cm] at (0,0) {};
\node[anchor=west] at (0.3,0) {QK-norm RoPE attention};
\node[pol, minimum width=0.45cm, minimum height=0.3cm] at (3.6,0) {};
\node[anchor=west] at (3.9,0) {added by the Polar Transformer};
\end{scope}
\end{tikzpicture}}
\caption{\textbf{Transformer and Polar Transformer.} Orange marks
what the Polar Transformer adds.}
\label{fig:model_comparison}
\end{figure}

\section{Conclusion}
\label{sec:Conclusion}

We showed that the core components of the Transformer---attention, residual
connections, normalization, and rotary encodings---arise from a single
state-estimation problem on the hypersphere. Modeling each token as a noisy
reading of a shared, rotating direction at an independent or shared-level
magnitude induces polar coordinates and von Mises--Fisher likelihoods,
yielding the closed-form Polar Filter. This framework unifies QK
normalization, unnormalized values, and Exclusive Self-Attention within one
model. We also introduced the Polar Transformer, which retains the elements
of the Polar Filter that the standard Transformer omits.

More broadly, this work illustrates the usefulness of viewing sequence models
as state estimators. Starting from a probabilistic and dynamical model,
deriving its estimator, and instantiating that estimator as a layer offers a
systematic way to interpret existing components and guide architectural
choices. Several directions follow naturally from this framework: extending
the polar observation model to recurrent state-space models;
introducing process noise on the sphere to derive lag-dependent precision
decay and forgetting; and generalizing from single directions to richer
manifolds, such as subspaces, or to anisotropic directional likelihoods.

\bibliographystyle{plainnat}
\bibliography{references}

\newpage

\appendix

\section*{Appendix Contents}

\newcommand{\appentry}[2]{%
  \noindent\textbf{Appendix~\ref{#1}: \nameref{#1}}\hfill\pageref{#1}\\
  \hspace*{1em}\parbox[t]{\dimexpr\linewidth-1em}{\textit{#2}}\par\medskip}

\appentry{sec:Implementation}{Real-valued computation, normalization,
multiple heads on one sphere or one per head, fused attention,
initialization, and pseudocode for the Polar Transformer.}

\appentry{sec:architectural_implications}{Magnitude as confidence, a
possible reason for the split between keys and values, the growth of
residual-stream norms with depth, normalization, and the residual connection
as iterated estimation.}

\appentry{sec:extensions}{Derives the Learned-Scatter Polar Transformer, which
generalizes the Polar Transformer by interpolating between token-specific
and shared-level magnitudes for both keys and values. Also extends the model to signed amplitudes.}

\newpage

\section{Implementation}
\label{sec:Implementation}

\paragraph{Complex arithmetic.}
The polar model is formulated in $\mathbb C^d$ and computed in
$\mathbb R^{2d}$, where a complex map $\boldsymbol W_r+i\boldsymbol W_i$ is
the real block matrix
$\left[\begin{smallmatrix}\boldsymbol W_r&-\boldsymbol W_i\\\boldsymbol W_i&\boldsymbol W_r\end{smallmatrix}\right]$.
Since the layer's input and output are real, all projections are ordinary
linear layers, $\mathcal L_q,\mathcal L_k,\mathcal L_v:\mathbb R^d\to\mathbb R^{2d}$
and $\mathcal L_o:\mathbb R^{2d}\to\mathbb R^d$. The Polar Filter uses only
the real inner product $\mathrm{Re}(\boldsymbol a^\dagger\boldsymbol b)$,
the Euclidean inner product in $\mathbb R^{2d}$, so with
$\tilde{\boldsymbol q}_i=r\tilde{\boldsymbol u}_{q,i}$ and
$\tilde{\boldsymbol k}_j=r\tilde{\boldsymbol u}_{k,j}$ the rotated RMSNorm
outputs, one real matrix product gives the agreements
$c_{ij}=\mathrm{Re}(\tilde{\boldsymbol q}_i^\dagger\tilde{\boldsymbol k}_j)/r^2$.
The frequencies come in conjugate pairs $\pm\omega_k$, so the system matrix
is real and the transport $e^{i\boldsymbol\Omega t}$ is a set of
two-dimensional rotations of real--imaginary pairs, as in RoPE. RoPE gives
each of the $d$ planes of $\mathbb R^{2d}$ its own frequency; here, each
frequency is shared by two planes, which rotate in opposite directions, so
the layer uses $d/2$ distinct frequencies. A plane rotating at $-\omega$ is a
plane rotating at $\omega$ with its coordinates swapped, which the
projections can absorb, so this is RoPE with each frequency used twice.

\paragraph{Normalization.}
RMSNorm maps each column of $\boldsymbol Q$ and $\boldsymbol K$ to the
sphere of radius $r=g\sqrt{2d}$, with a scalar gain $g$ per head, so its
outputs are $r\hat{\boldsymbol u}$
(Section~\ref{sec:Transformer_as_RT_Filter}); magnitudes are measured before
normalization. The values, with shared-level magnitudes, are aggregated
unnormalized; the Learned-Scatter Polar Transformer of Appendix~\ref{sec:extensions} also
normalizes them. A per-feature gain on the queries and keys would replace the cosine by a
weighted inner product, corresponding to anisotropic noise across features;
even when equal within each rotated pair (so it commutes with the
rotations), this breaks the isotropic vMF likelihood and
couples frequency pairs on the sphere. We therefore use a scalar gain.

\paragraph{Multiple heads.}
\label{sec:multihead_attention}
Partitioning the conjugate frequency pairs into $H$ blocks makes the
transport block diagonal,
$\boldsymbol\Phi(\tau)=\bigoplus_he^{i\boldsymbol\Omega^{(h)}\tau}$, so it
never mixes heads and preserves the norm of each block. We normalize each
block separately, placing the direction on a product of spheres. Like the
magnitude of Section~\ref{sec:rt_model}, the block norms are never changed
by the dynamics, and they act as per-head magnitudes. Each block has its own
readings, with its own noise levels, scatters, and prior strengths, and the
likelihood separates across heads,
\[
\mathcal L(\boldsymbol u)
=\frac12\sum_{h=1}^{H}\sum_{j\le i}\kappa_{ij}^{(h)}
\big\|\boldsymbol u^{(h)}-\hat{\boldsymbol u}_{ij}^{(h)}\big\|^2+\text{const}.
\]
Each head is then a single-head Polar Filter, with its own robust weights,
attention matrix, consensus, and step; the per-head steps are concatenated,
mapped through one $\boldsymbol W_o$, and added through one residual
connection. This is the per-head query--key normalization common in
practice. In Algorithm~\ref{alg:polar_transformer}, the normalization, the
magnitudes (line~\ref{line:magnitudes}), and the step
(line~\ref{line:update}) are applied per head.

\paragraph{A single sphere.}
Normalizing the whole vector across heads instead places the direction on
one sphere with one magnitude per token. Each head's block of the normalized vector then has a norm
$r_i^{(h)}=\|\hat{\boldsymbol u}_i^{(h)}\|$, and these satisfy
$\sum_h(r_i^{(h)})^2=r^2$. A block's norm measures the share of the token
that lies in that head's subspace. Because the shares add up to a fixed
total, one head can gain only at the others' expense, so the block norms are
part of the direction, not separate magnitudes. With noise variance $\sigma_h^2$ and
robust weights $w_{ij}^{(h)}$ in head $h$, the log-likelihood is a sum over
heads,
\[
\sum_{h=1}^{H}\Big(S_i^{(h)}\big\langle\boldsymbol u^{(h)},
\bar{\boldsymbol u}_i^{(h)}\big\rangle
-a_i^{(h)}\big\|\boldsymbol u^{(h)}\big\|^2\Big),
\qquad
S_i^{(h)}=\sum_{j\le i}w_{ij}^{(h)}\kappa_{ij}^{(h)},
\]
where $\kappa_{ij}^{(h)}=m_i\hat m_j/\sigma_h^2$ and
$a_i^{(h)}=(m_i^2/2\sigma_h^2)\sum_{j\le i}w_{ij}^{(h)}$. On a product of spheres, the second term is constant; on one sphere, it is not, and the
constraint $\|\boldsymbol u\|=r$ is shared. The estimate is then
\[
\hat{\boldsymbol u}_i^{+(h)}
=\tfrac{1}{2} S_i^{(h)}\,\bar{\boldsymbol u}_i^{(h)} / {\big(a_i^{(h)}+\mu_i\big)},
\]
where the Lagrange multiplier $\mu_i > -\min_h a_i^{(h)}$ is the unique
scalar per token satisfying $\sum_h\|\hat{\boldsymbol u}_i^{+(h)}\|^2=r^2$. Each head's consensus enters in proportion to its evidence, discounted by its
noise, and the heads compete for one norm budget. A per-head softmax
normalizes each head's weights to one and discards $S_i^{(h)}$, so standard
multi-head attention
on one sphere treats every head as equally informed; this coupling
is what it ignores. The tangent projection couples the heads as well:
because each head averages readings with different amounts of norm in its own
block, the concatenated step has a first-order radial component unless all
heads share one attention matrix. A product of spheres removes both
couplings, which is one reason we use it. Pre-normalizing the residual stream approximates this single sphere when the
projections are close to unitary. Per-head query--key normalization restores a product of spheres for the queries and keys, leaving
pre-normalization only to bound the values' magnitudes, which the proximal
step does in the Polar Transformer.

\paragraph{Fused attention.}
Under the exponential kernel, the key precision depends only on the key's
token, and the logits split into a per-key bias and a dot product with
rescaled keys,
\[
-\frac{2\kappa^k_j}{d}\big(1-c_{ij}\big)
=-\frac{2\kappa^k_j}{d}
+\mathrm{Re}\Big(\tilde{\boldsymbol q}_i^\dagger\,\frac{2\kappa^k_j}{d\,r^2}\tilde{\boldsymbol k}_j\Big).
\]
Appending $1$ to each query and $-2\kappa^k_j/d$ to each key absorbs the
bias, so a standard fused attention kernel applies, with unit scale, and
aggregates the rotated values. With the Student-$t$ weight, and in the
Learned-Scatter Polar Transformer, the logits do not
split, but they are an elementwise function of the agreements and the
magnitudes, so a tiled kernel in the style of FlashAttention \citep{dao2022flashattention} computes them at
the same memory cost. Every robust log-weight is at most zero, so the
evidence $W_i$ is a row sum that does not overflow.

\paragraph{Initialization.}
We initialize $\eta_k$ small, so that the key test initially ignores key
magnitudes and starts as standard QK-normalized attention, and the prior
isotropic, $c_r=c_t=1$, so that the step starts isotropic. The key noise
sets the temperature of attention: since the agreement is a cosine, the
content logit spans $4\kappa^k/d$. When using the Student-$t$ kernel
(Eq.~\ref{eq:student_t_weight}), the degrees of freedom $\nu_t$ start in a
quasi-Gaussian regime, which avoids rejecting tokens while the queries and
keys are untrained. The noise levels, prior strengths, and $\nu_t$ are
learned through a softplus, and value magnitudes are floored at a small
constant before division when forming $\hat{\boldsymbol u}_{v,i}$.

\paragraph{Tangential versus geodesic steps.}
The directional step moves along the tangent space, removing the attention
output's component along the token's value, as in XSA
\citep{zhai2026exclusiveselfattention}
(Figure~\ref{fig:transformer_manifold_geometry}). A tangent vector of length
$\delta$ added at radius $m$ gives radius
$\sqrt{m^2+\delta^2} \approx m + \delta^2/(2m)$, changing the magnitude only
at second order in $\delta/m$. A geodesic step, spherical linear
interpolation toward the consensus, as in GeoNorm
\citep{zheng2026geonormunifyprenormpostnorm}, avoids this and agrees with the
tangential step to first order. We keep the tangential step because the
update is added to the residual stream through $\boldsymbol W_o$: the stream
does not lie on a sphere, so the geodesic step's exact magnitude would not
survive the projection, while the tangential step keeps the update linear in
the attention output.

\begin{figure}[H]
\centering
\begin{subfigure}[b]{0.44\linewidth}
    \centering
    \includegraphics[
        width=\linewidth,
        trim={0.0cm 0.0cm 3.5cm 1.0cm},
        clip
    ]{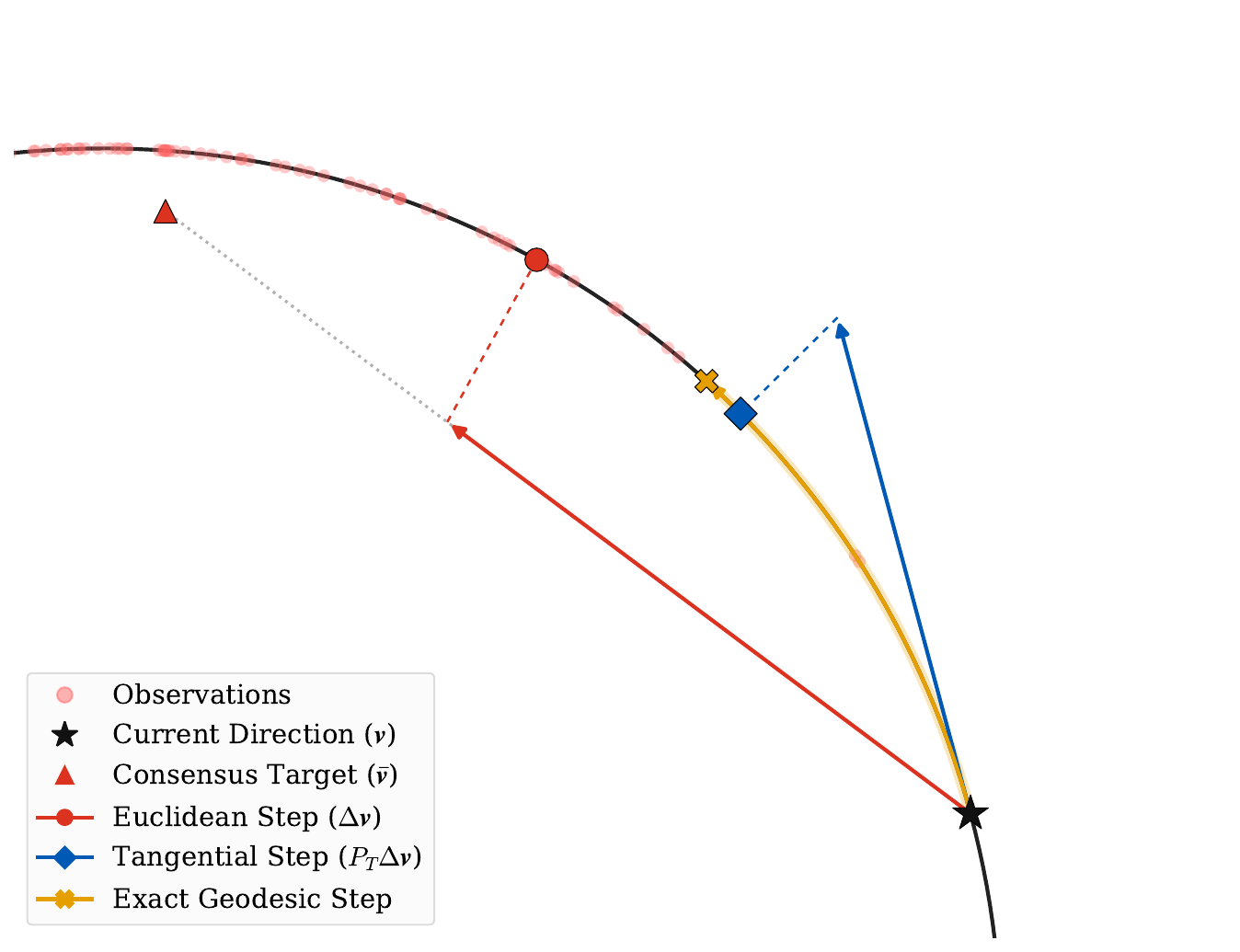}
    \caption{Directional steps on the circle}
    \label{fig:transformer_manifold_geometry_a}
\end{subfigure}
\hfill
\begin{subfigure}[b]{0.44\linewidth}
    \centering
    \includegraphics[
        width=\linewidth,
        trim={0.0cm 0.0cm 3.5cm 1.0cm},
        clip
    ]{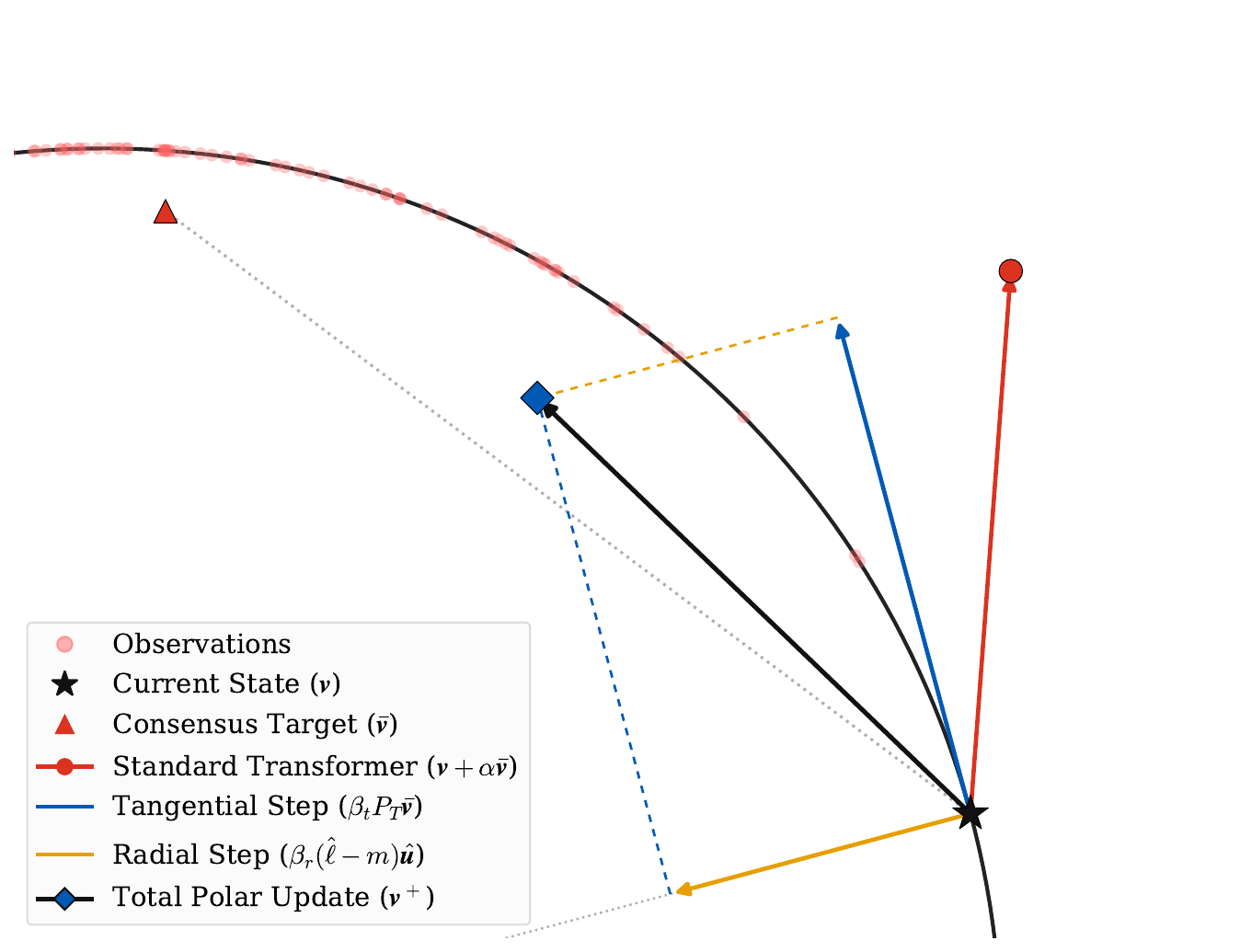}
    \caption{Full residual update}
    \label{fig:transformer_manifold_geometry_b}
\end{subfigure}
\caption{
Two-dimensional geometry of the Polar Transformer update.
\textbf{(a)}~Directional steps on the circle: stepping toward the consensus
$\bar{\boldsymbol v}$ in ambient Euclidean space ($c_r=c_t$, red) moves off
the sphere, whereas the tangential step (blue) agrees with the exact
geodesic (orange) to first order.
\textbf{(b)}~Full residual update: the Polar Transformer combines the
tangential step (blue) with a radial step toward the level $\hat{\ell}$
(orange) to form $\boldsymbol v^+$ (black), whereas the standard Transformer
adds $\bar{\boldsymbol v}$ whole (red), pushing the state outward and
growing its norm.
}
\label{fig:transformer_manifold_geometry}
\end{figure}

\clearpage

\subsection{Algorithm}
\label{sec:Algorithm}

Algorithm~\ref{alg:polar_transformer} gives the single-head Polar Transformer layer of
Section~\ref{sec:polar_transformer}, and
Figure~\ref{fig:polar_transformer} shows it as a block diagram; there,
$\boldsymbol P_T(\boldsymbol u)\boldsymbol w=\boldsymbol w-\boldsymbol u\,\mathrm{Re}(\boldsymbol u^\dagger\boldsymbol w)$
for a unit vector $\boldsymbol u$.

\begin{algorithm}[H]
\caption{Polar Transformer (single head)}
\label{alg:polar_transformer}
\begin{algorithmic}[1]
\REQUIRE Sequence $\boldsymbol Z\in\mathbb R^{d\times N}$ at positions $t_1,\dots,t_N$
\REQUIRE Projections $\boldsymbol W_q,\boldsymbol W_k,\boldsymbol W_v$ ($d\to2d$) and $\boldsymbol W_o$ ($2d\to d$); frequencies $\boldsymbol\omega$ in conjugate pairs
\REQUIRE Normalization gain $g$; key noise $\eta_k^2,\rho_k^2$; DoF parameter $\nu_t$; prior strengths $c_r,c_t$
\ENSURE Updated sequence $\boldsymbol Z_{\mathrm{out}}\in\mathbb R^{d\times N}$
\STATE $\boldsymbol Q,\ \boldsymbol K,\ \boldsymbol V \gets \boldsymbol W_q\boldsymbol Z,\ \boldsymbol W_k\boldsymbol Z,\ \boldsymbol W_v\boldsymbol Z$ \COMMENT{project}
\STATE $\boldsymbol m_k,\ \boldsymbol m_v \gets$ column norms of $\boldsymbol K,\ \boldsymbol V$ \COMMENT{magnitudes}\label{line:magnitudes}
\STATE $\hat{\boldsymbol Q},\ \hat{\boldsymbol K} \gets \mathrm{RMSNorm}_g(\boldsymbol Q,\boldsymbol K)$ \COMMENT{columns of norm $r=g\sqrt{2d}$}
\STATE $\tilde{\boldsymbol Q},\ \tilde{\boldsymbol K},\ \tilde{\boldsymbol V} \gets \tilde{\boldsymbol\Phi}^{-}\odot(\hat{\boldsymbol Q},\hat{\boldsymbol K},\boldsymbol V)$ \COMMENT{rotate (RoPE); values unnormalized}
\STATE $\boldsymbol\kappa^k[j] \gets \boldsymbol m_k[j]^2/(\eta_k^2+\rho_k^2\boldsymbol m_k[j]^2)$ \COMMENT{keys}\label{line:precision}
\STATE $\boldsymbol D^2[i,j] \gets 2\boldsymbol\kappa^k[j]\big(1-\mathrm{Re}(\tilde{\boldsymbol Q}[:,i]^{\dagger}\tilde{\boldsymbol K}[:,j])/r^2\big)$ \COMMENT{directional distance}\label{line:distance}
\STATE $\boldsymbol\Lambda \gets -(\nu_t+1)\log(1+\boldsymbol D^2/\nu_t d)+\boldsymbol M_{\mathrm{causal}}$ \COMMENT{robust log-weights, $\le0$}\label{line:logits}
\STATE $\boldsymbol W[i] \gets \sum_je^{\boldsymbol\Lambda[i,j]},\quad \boldsymbol\Pi \gets e^{\boldsymbol\Lambda}/\boldsymbol W$ \COMMENT{evidence and attention}
\STATE $\boldsymbol\beta_r,\ \boldsymbol\beta_t \gets \boldsymbol W/(c_r+\boldsymbol W),\ \boldsymbol W/(c_t+\boldsymbol W)$ \COMMENT{evidence shares}
\STATE $\bar{\boldsymbol V} \gets \tilde{\boldsymbol\Phi}^{+}\odot\big(\tilde{\boldsymbol V}\boldsymbol\Pi^{\top}\big)$ \COMMENT{estimate}\label{line:consensus}
\STATE $\hat{\boldsymbol U}_v \gets \boldsymbol V/\boldsymbol m_v,\quad \hat{\boldsymbol \ell}[i] \gets \mathrm{Re}\big(\hat{\boldsymbol U}_v[:,i]^{\dagger}\bar{\boldsymbol V}[:,i]\big)$ \COMMENT{level}
\STATE $\Delta\boldsymbol V[:,i] \gets \boldsymbol\beta_t[i]\,\boldsymbol P_T(\hat{\boldsymbol U}_v[:,i])\,\bar{\boldsymbol V}[:,i]+\boldsymbol\beta_r[i]\big(\hat{\boldsymbol \ell}[i]-\boldsymbol m_v[i]\big)\hat{\boldsymbol U}_v[:,i]$ \COMMENT{proximal step}\label{line:update}
\STATE $\boldsymbol Z^{+} \gets \boldsymbol Z+\boldsymbol W_o\Delta\boldsymbol V$ \COMMENT{residual update}
\STATE $\boldsymbol Z_{\mathrm{out}} \gets \boldsymbol Z^{+}+\mathrm{FFN}\big(\mathrm{RMSNorm}(\boldsymbol Z^{+})\big)$ \COMMENT{feedforward}
\RETURN $\boldsymbol Z_{\mathrm{out}}$
\end{algorithmic}
\end{algorithm}

\vspace{1cm}

\begin{figure}[H]
  \centering
  \resizebox{1.0\textwidth}{!}{
\begin{tikzpicture}[
  font=\small, >={Stealth[length=2mm]},
  box/.style={draw, rounded corners=2pt, align=center, minimum height=1.0cm, inner sep=3pt, line width=0.6pt},
  sbox/.style={box, minimum height=0.7cm},
  std/.style={box, fill=gray!10, draw=gray!60},
  rt/.style={box, fill=orange!16, draw=orange!75!black},
  mag/.style={box, fill=blue!7, draw=blue!55!black},
  arr/.style={->, line width=0.7pt, draw=black!75},
  marr/.style={->, line width=0.7pt, draw=blue!55!black},
  skip/.style={->, dashed, line width=0.6pt, draw=black!45},
  lab/.style={font=\scriptsize, inner sep=1pt},
]
\node[std, minimum height=0.8cm] (Z) at (0,-1.7) {$\boldsymbol Z$};
\node[std] (qk)   at (1.6,0)   {$\boldsymbol W_q,\boldsymbol W_k$\\{\scriptsize $d\!\to\!2d$}};
\node[std] (nrqk) at (4.1,0)   {normalize\\rotate $\tilde{\boldsymbol\Phi}^{-}$};
\node[std] (C)    at (6.75,0)  {dot product\\$\mathrm{Re}(\tilde{\boldsymbol Q}^{\dagger}\tilde{\boldsymbol K})$};
\node[rt]  (L)    at (9.55,0)  {robust logits\\$\log w(\boldsymbol D^2)$};
\node[std] (sm)   at (12.25,0) {softmax $\boldsymbol\Pi$\\evidence $\boldsymbol W$};

\node[std] (v)     at (1.6,-3.4)   {$\boldsymbol W_v$\\{\scriptsize $d\!\to\!2d$}};
\node[rt]  (rv)    at (4.1,-3.4)   {rotate\\$\tilde{\boldsymbol\Phi}^{-}$};
\node[std] (agg)   at (6.75,-3.4)  {aggregate\\$\tilde{\boldsymbol V}\boldsymbol\Pi^{\top}$};
\node[rt]  (crot)  at (9.45,-3.4)  {counter-rotate\\$\tilde{\boldsymbol\Phi}^{+}$};
\node[rt]  (step)  at (12.9,-3.4)  {proximal step\\$\boldsymbol\beta_t\boldsymbol P_T(\hat{\boldsymbol U}_v)\bar{\boldsymbol V}$\\$+\,\boldsymbol\beta_r(\hat{\boldsymbol \ell}-\boldsymbol m_v)\hat{\boldsymbol U}_v$};
\node[std] (wo)    at (16.2,-3.4) {$\boldsymbol W_o$\\{\scriptsize $2d\!\to\!d$}};
\node[std] (out)   at (16.2,-5.3)  {residual update\\$\boldsymbol Z^{+}=\boldsymbol Z+\Delta\boldsymbol Z$};

\node[mag, sbox] (mk)  at (5.3,1.6)  {$\boldsymbol m_k=\|\boldsymbol K\|$};
\node[mag, sbox] (kap) at (9.55,1.6) {key precision $\boldsymbol\kappa^k(\boldsymbol m_k)$};

\draw[arr] (-1.3,-1.7) -- (Z.west);
\draw[arr] (Z.north) |- (qk.west);
\draw[arr] (Z.south) |- (v.west);
\draw[arr] (qk) -- (nrqk);
\draw[arr] (nrqk) -- node[lab, above] {$\tilde{\boldsymbol Q},\tilde{\boldsymbol K}$} (C);
\draw[arr] (C) -- (L);
\draw[arr] (L) -- (sm);

\draw[arr] (v) -- (rv);
\draw[arr] (rv) -- node[lab, above] {$\tilde{\boldsymbol V}$} (agg);
\draw[arr] ([xshift=-5mm]sm.south) -- ([xshift=-5mm]sm.south |- 0,-2.35) -- node[lab, above] {$\boldsymbol\Pi$} (6.75,-2.35) -- (agg.north);
\draw[arr] ([xshift=5mm]sm.south) -- node[lab, right, pos=0.4] {$\boldsymbol W$} ([xshift=5mm]sm.south |- step.north);
\draw[arr] (agg) -- (crot);
\draw[arr] (crot) -- node[lab, above] {$\bar{\boldsymbol V}$} (step);
\draw[arr] (step) -- node[lab, above] {$\Delta\boldsymbol V$} (wo);
\draw[arr] (wo.south) -- node[lab, right] {$\Delta\boldsymbol Z$} (out.north);
\draw[arr] (out.east) -- ++(0.7,0);

\fill[black!60] (3.05,-3.4) circle (1.3pt);
\draw[skip] (3.05,-3.4) -- (3.05,-4.7) -| node[lab, above, pos=0.25] {$\boldsymbol V$ (current state)} (step.south);

\fill[black!60] (-0.75,-1.7) circle (1.3pt);
\draw[skip] (-0.75,-1.7) -- (-0.75,-6.1) -- node[lab, above, pos=0.5] {$\boldsymbol Z$ (residual stream)} (out.south |- 0,-6.1) -- (out.south);

\fill[blue!55!black] (2.85,0) circle (1.3pt);
\draw[marr] (2.85,0) |- (mk.west);
\draw[marr] (mk) -- (kap);
\draw[marr] (kap.south) -- node[lab, right] {$\boldsymbol D^2$} (L.north);

\begin{scope}[shift={(2.6,-7.0)}, font=\scriptsize]
  \node[std, minimum height=0.3cm, minimum width=0.45cm, inner sep=0pt] at (0,0) {};
  \node[anchor=west] at (0.3,0) {as in RoPE attention};
  \node[rt, minimum height=0.3cm, minimum width=0.45cm, inner sep=0pt] at (3.9,0) {};
  \node[anchor=west] at (4.2,0) {modified by the Polar Transformer};
  \node[mag, minimum height=0.3cm, minimum width=0.45cm, inner sep=0pt] at (8.3,0) {};
  \node[anchor=west] at (8.6,0) {magnitudes};
\end{scope}
\end{tikzpicture}
}

\caption{\textbf{A single Polar Transformer attention layer.} Key magnitudes
(blue) set each key's precision in the test. Unnormalized values are rotated,
aggregated, and rotated back, and the proximal step turns each token's value
direction and moves its magnitude toward the level, each by its share of the
evidence $\boldsymbol W$.}

  \label{fig:polar_transformer}
\end{figure}

\newpage

\section{Architectural Interpretation and Implications}
\label{sec:architectural_implications}

\paragraph{Magnitude as confidence.}
Since the direction is shared, differences in token magnitude belong to the
observations, not to the state, and the Polar Filter reads them as
confidence---consistent with work showing that embedding norms encode
information gain \citep{oyama2023norm}. Whether the magnitudes are token-specific or share a level sets
whether the precision is quadratic or linear in the magnitude; the noise sets
the rest, with fixed-variance noise favoring larger readings and noise
proportional to the magnitude giving a flat precision. The standard
Transformer takes one end for each reading. Normalized keys, which ensure
that no key is unselectable \citep{brody2023expressivity}, are token-specific
with flat precision; unnormalized values, whose norms substantially modulate
the output \citep{kobayashi2020attention}, share a level, with precision
linear in the magnitude. Learned representations, whose norms vary widely
across tokens and grow with depth \citep{kobayashi2020attention,xiong2020layer},
need not sit at either end: the Polar Transformer lets key magnitudes set the
key precision through $\eta_k$, and the Learned-Scatter Polar Transformer of
Appendix~\ref{sec:extensions} learns the scatter of both readings.

\paragraph{A possible reason for the split.}
One interpretation, which the derivation does not establish, is that keys and
values do different jobs. Keys only feed the test, so a key's magnitude need
only say how confident its reading is, and token-specific magnitudes are the
natural model. Values carry the state, magnitude included, and the context
can inform a token's own magnitude only if magnitudes are shared. On this
reading, the value side's question is how far the context may change a
token's magnitude: the standard Transformer's sum answers ``completely''
($\lambda_v=0$), and XSA answers ``not at all'', as does $\lambda_v=1$. In a
memory, the split is natural, since addresses and contents play different
roles \citep{bricken2021attention}. The Learned-Scatter Polar Transformer allows this prediction to be tested by learning the values of $\lambda_k$ and $\lambda_v$.

\paragraph{Residual-stream magnitude across depth.}
Residual-stream magnitude grows across the layers of Pre-Norm Transformers
\citep{xiong2020layer,sun2025curse}. The standard Transformer adds the
estimate whole; its radial part is the level $\hat \ell_i$, so with
$\boldsymbol W_o$ scaling the step by a small $\alpha$, the updated magnitude is
$m_i^+\approx\hat m_i+\alpha\hat \ell_i$, and consistent evidence adds the level
at every layer, even when the attention output agrees with the current
direction. The proximal step instead moves each magnitude toward the level,
$m_i^+\approx\hat m_i+\beta_{r,i}(1-\lambda_i)(\hat\ell_i-\hat m_i)$, and at
$c_r\to\infty$ or $\lambda_i=1$ leaves it unchanged to first order. Since the Pre-Norm update is computed from the normalized stream, its size
does not grow with the magnitude for weights of comparable scale, so each
layer rotates the state less as the magnitude grows---consistent with deep
Pre-Norm layers approaching the identity \citep{sun2025curse}.

\paragraph{Normalization.}
In the model, a reading's magnitude carries its confidence, so it should be
used rather than normalized away. The Polar Transformer measures magnitudes
before normalizing, normalizes queries and keys within each head only to
compare directions, and leaves values unnormalized; normalizing the whole
vector instead couples the heads through a shared norm budget
(Appendix~\ref{sec:multihead_attention}). What keeps magnitudes from growing
is then the proximal step, which moves each magnitude only as far as the
prior and the token's own reading allow, rather than a normalization that
discards them.

\paragraph{Depth as iterated estimation.}
Each layer performs one step of the Polar Filter: attention evaluates the
robust weights at the current estimate and forms the consensus, and the
residual connection steps toward it. With the observations held fixed, this
is one step of iteratively reweighted least squares (IRLS); repeating it with the
weights re-evaluated is mean shift on the sphere
\citep{cheng1995mean,kobayashi2010vmf}, which converges to the mode of the
transported observations nearest the query. The standard Transformer adds
the attention output instead of stepping toward it
(Figure~\ref{fig:transformer_manifold_geometry}).

\paragraph{The feedforward layer.}
The feedforward layer (FFN) lies outside our derivation. Within this framework, each layer can be viewed as
alternating an estimation step (attention) with a learned nonlinear
transition (the FFN). Pre-normalization before the FFN then projects its
input onto the sphere so that the transition depends only on the
direction---consistent with our model, where dynamics act on the direction
alone. Because the FFN output is added back to the unnormalized residual
stream, tokens with larger magnitude (higher confidence) are still rotated
less in angle by a given update, even though the FFN does not read the
magnitude.

\newpage

\section{Extensions}
\label{sec:extensions}

\subsection{The Learned-Scatter Polar Transformer}

Section~\ref{sec:rt_model} introduced a hierarchical magnitude model
$m_j\sim\mathcal N(\ell,\zeta^2)$ but fixed its two limits: token-specific
magnitudes ($\zeta\to\infty$) for keys and a shared level ($\zeta\to0$) for
values. Here, we keep $\zeta^2$ finite and learned for both, deriving a single
joint objective for the robust test and estimation that contains the Polar
Transformer as a special case.

\paragraph{Model.}
For a candidate direction $\boldsymbol u$, token $j$ produces the reading
$\hat{\boldsymbol x}_j=m_j\boldsymbol u+\boldsymbol\varepsilon_j$ with measurement noise
$\boldsymbol\varepsilon_j\sim\mathcal N(\boldsymbol 0,\sigma_j^2\boldsymbol I)$
and latent magnitude $m_j=\ell+\delta_j$, $\delta_j\sim\mathcal N(0,\zeta^2)$.
The noise variance $\sigma_j^2=\eta^2+\rho^2\hat m_j^2$ is evaluated at the
observed magnitude and held fixed within a step, and we write
$c_j=\langle\boldsymbol u,\hat{\boldsymbol u}_j\rangle$. Our goal is to
estimate token $i$'s direction and its own magnitude: other tokens inform the direction and the shared level $\ell$, and inform token $i$'s magnitude only through $\ell$, so we estimate $\boldsymbol u$ and $\ell$ from all readings and then combine $\ell$ with token $i$'s own reading.

\paragraph{Objective.}
Unlike Section~\ref{sec:likelihood_decomposition}, which conditioned on
$\hat m_j$ and kept only the directional likelihood, a finite scatter
$\zeta^2$ means the observed magnitude also carries evidence about $\ell$, so
we retain both factors of $p(\hat m_j,\hat{\boldsymbol u}_j\mid\boldsymbol u,\ell)=p(\hat m_j\mid\ell)\,p(\hat{\boldsymbol u}_j\mid\hat m_j,\boldsymbol u,\ell)$.
Along $\boldsymbol u$, the observed magnitude is the latent magnitude plus
radial noise, so it is Gaussian about $\ell$ with variance $\sigma_j^2+\zeta^2$.
Given $\hat m_j$, the latent magnitude has posterior mean
\[
m_j^*=\lambda_j\hat m_j+(1-\lambda_j)\ell,
\qquad
\lambda_j=\frac{\zeta^2}{\sigma_j^2+\zeta^2},
\]
and plugging it in as in Eq.~\ref{eq:pooled_concentration} gives a vMF
direction with concentration $\kappa_j=\hat m_jm_j^*/\sigma_j^2$. The
negative log-likelihood of the reading is then $D_j^2/2$, up to a constant,
with
\begin{equation}
D_j^2=2\kappa_j(1-c_j)+\frac{(\hat m_j-\ell)^2}{\sigma_j^2+\zeta^2},
\qquad
\kappa_j=\frac{\hat m_j}{\sigma_j^2} \bigl(\lambda_j\hat m_j+(1-\lambda_j)\ell\bigr).
\label{eq:rt_distance}
\end{equation}
The two approximations are those of the main text: the plug-in concentration,
and treating the observed magnitude as the radial component, which is
accurate when the magnitude is large compared with the length of the noise vector, $\sqrt{2d}\sigma_j$.

The first term in $D_j^2$ measures disagreement in direction, the second disagreement in
magnitude. The radial term is divided by the scatter as well as the noise.
With token-specific magnitudes, $\zeta\to\infty$, it vanishes: magnitudes vary
freely, so an unusual one is no evidence against a reading, and $D_j^2$ is the
key test of the Polar Transformer, with $\kappa_j=\hat m_j^2/\sigma_j^2$. With a
shared level, $\zeta\to0$, it is the full radial residual, and $D_j^2 = \|\hat{\boldsymbol x}_j-\ell\boldsymbol u\|^2/\sigma_j^2$, the Euclidean test of RFA.

The robust test applies the penalty of Section~\ref{sec:robust} to this
distance, $w_j=\rho'(D_j^2)$, and, as in the main text, the weights are then
held fixed. The direction and the level minimize
$\mathcal L_w(\boldsymbol u,\ell)=\tfrac12\sum_jw_jD_j^2$, and holding the other at the current state, both minimizers have closed forms.

\paragraph{The update.}
One step of the filter takes a parallel block-coordinate step on $\mathcal L_w(\boldsymbol u,\ell)$ from the current state, replacing $\hat{\boldsymbol x}_i=\hat m_i\hat{\boldsymbol u}_i$ by a new magnitude and a new direction.

\emph{Direction.} The concentration $\kappa_j$ does not depend on
$\boldsymbol u$, so $\mathcal L_w$ is linear in the agreements $c_j$, as in
Section~\ref{sec:likelihood_decomposition}. Its minimizer on the sphere is therefore the
normalized weighted average of the readings, as in Eq.~\ref{eq:consensus},
with weights $A_j=w_j\kappa_j/\sum_{j'}w_{j'}\kappa_{j'}$. The step toward it
is the consensus of Eq.~\ref{eq:directional_update},
\[
\Delta\boldsymbol u_i
=\boldsymbol P_T(\hat{\boldsymbol u}_i)\,\bar{\boldsymbol u}_i,
\qquad
\bar{\boldsymbol u}_i=\sum_jA_j\hat{\boldsymbol u}_j.
\]

\emph{Magnitude.} The level enters both terms of Eq.~\ref{eq:rt_distance},
through $\kappa_j$ and through the radial residual. Using
$1/(\sigma_j^2+\zeta^2)=(1-\lambda_j)/\sigma_j^2$, the two derivatives combine
into
\[
\frac{\partial D_j^2}{\partial\ell}
=\frac{2(1-\lambda_j)}{\sigma_j^2}\bigl((1-c_j)\hat m_j-(\hat m_j-\ell)\bigr)
=2b_j(\ell-\hat m_jc_j),
\qquad
b_j=\frac{1}{\sigma_j^2+\zeta^2},
\]
and setting $\partial\mathcal L_w/\partial\ell = \sum_j w_j b_j(\ell - \hat m_j c_j) = 0$ gives the level
\[
\hat\ell_i= \big( \sum_jw_jb_j \big)^{-1} \sum_jw_jb_j\hat m_jc_j,
\]
an average of the readings' lengths along the current direction, each
weighted by how precisely it measures the level. The token's own magnitude is then estimated from its own reading and this
level, as $m_i^*$ above, giving the radial target
$\lambda_i\hat m_i+(1-\lambda_i)\hat\ell_i$ and discrepancy
$\Delta m_i=(1-\lambda_i)(\hat\ell_i-\hat m_i)$. Moving toward it by the
share $\beta_{r,i}$ that the prior leaves gives the new magnitude
$m_i^+=\hat m_i+\beta_{r,i}\Delta m_i$.

\emph{Full update.} Reassembling the updated magnitude and direction directly
would scale the tangent-space turn $\beta_{t,i}\Delta\boldsymbol u_i$ by the
token's own radius $m_i^+$, whereas the exact shared-level posterior
(Eq.~\ref{eq:pooled_step}) scales it by the readings' mean magnitude
$\bar m_i=\sum_j\pi_{ij}\hat m_j$. To recover both limits exactly, we
interpolate the turn's scale between the two,
$\tilde m_i=\lambda_im_i^{+}+(1-\lambda_i)\bar m_i$, giving
\[
\hat{\boldsymbol x}_i^{+}
=\hat{\boldsymbol x}_i
+\beta_{t,i}\,\tilde m_i\,\Delta\boldsymbol u_i
+\beta_{r,i}\,\Delta m_i\,\hat{\boldsymbol u}_i.
\]

\emph{Limits.} At the shared-level end ($\lambda_j=0$), the precision
$\kappa_j=\hat m_j\ell/\sigma_j^2$ is linear in the observed magnitude; under
common noise, $A_j\propto w_j\hat m_j$, so
$\tilde m_i\bar{\boldsymbol u}_i=\bar m_i\bar{\boldsymbol u}_i=\bar{\boldsymbol v}_i$,
and the update reduces to Eq.~\ref{eq:pooled_step} exactly. At the
token-specific end ($\lambda_j=1$), $\kappa_j=\hat m_j^2/\sigma_j^2$, the
radial step vanishes ($m_i^+=\hat m_i$), and the direction turns at the
token's own magnitude ($\tilde m_i=\hat m_i$).

\paragraph{The standard Transformer and the Polar Transformer.}
The standard Transformer uses opposite limiting cases for keys and values.
For the keys, $\lambda_j^k\to1$ and $\eta_k=0$, so the noise is proportional
to the magnitude and $\kappa_{ij}=1/\rho_k^2$: the test depends only on the
cosine, as in QK normalization. For the values, $\lambda_j^v\to0$ and
$\rho_v=0$, so the magnitudes share a level under a common noise; their
precision is linear in their magnitude, and their consensus is the direction
of the sum of the unnormalized values. With $\eta_k$ free, value transport,
and the proximal step, this pairing is exactly the Polar Transformer.

\paragraph{Implementation.}
The layer applies the estimator twice. The level is unknown, so, like the
direction, it is evaluated at the current state of token $i$. In the key
test, the reading magnitude is $m_{k,j}$, the current state is the query, so
the query's magnitude supplies the level, $m_{q,i}$, and
$\lambda_j^k=\zeta_k^2/(\sigma_{k,j}^2+\zeta_k^2)$. Substituting these into
Eq.~\ref{eq:rt_distance} gives
\[
D_{ij}^2=2\kappa^k_{ij}(1-c_{ij})+\frac{(m_{k,j}-m_{q,i})^2}{\sigma_{k,j}^2+\zeta_k^2},
\qquad
\kappa^k_{ij}=\frac{m_{k,j}}{\sigma_{k,j}^2} \bigl(\lambda_j^km_{k,j}+(1-\lambda_j^k)m_{q,i}\bigr),
\]
and hence the robust weights $w_{ij}$. For the values, the reading magnitude
is $m_{v,j}$, the current state is the token's own value, so its magnitude
supplies the level, $m_{v,i}$, and
\[
\lambda_j^v=\frac{\zeta_v^2}{\sigma_{v,j}^2+\zeta_v^2},
\qquad
m^*_{ij}=\lambda_j^vm_{v,j}+(1-\lambda_j^v)m_{v,i},
\qquad
\kappa^v_{ij}=\frac{m_{v,j}m^*_{ij}}{\sigma_{v,j}^2},
\]
with noise variances $\sigma_{k,j}^2=\eta_k^2+\rho_k^2m_{k,j}^2$ and
$\sigma_{v,j}^2=\eta_v^2+\rho_v^2m_{v,j}^2$. The attention weights combine
the value precision with the key-derived robustness:
\[
A_{ij}=\operatorname{softmax}_{j,\,\mathrm{causal}}\bigl(\log\kappa^v_{ij}+\log w_{ij}\bigr).
\]
Computing the exact level
$\hat\ell_i=(\sum_{j\le i}w_{ij}b_j m_{v,j}c_{ij})/(\sum_{j\le i}w_{ij}b_j)$
directly would require a second attention pass over the values, because its
weights $w_{ij}b_j$ differ from the directional attention weights
$A_{ij}\propto w_{ij}\kappa^v_{ij}$. To avoid a second pass, note that any
positive weighting of the projections $m_{v,j}c_{ij}$ estimates the same
level; we therefore replace $b_j$ by
$\kappa^v_{ij}/m_{v,j}=m^*_{ij}/\sigma_{v,j}^2$. Dividing numerator and
denominator by the softmax normalizer $\sum_{j'}w_{ij'}\kappa^v_{ij'}$
cancels $m_{v,j}$ in the numerator and reuses the directional consensus
$\bar{\boldsymbol u}_i=\sum_{j\le i}A_{ij}\hat{\boldsymbol u}_{v,ij}$:
\[
\hat\ell_i
\approx \frac{\sum_{j\le i}w_{ij}(\kappa^v_{ij}/m_{v,j})\,m_{v,j}c_{ij}}{\sum_{j\le i}w_{ij}(\kappa^v_{ij}/m_{v,j})}
=\frac{\sum_{j\le i}A_{ij}c_{ij}}{\sum_{j\le i}A_{ij}/m_{v,j}}
=\bar m_i\big\langle\hat{\boldsymbol u}_{v,i},\bar{\boldsymbol u}_i\big\rangle,
\qquad
\bar m_i=\Big(\sum_{j\le i}\frac{A_{ij}}{m_{v,j}}\Big)^{-1}.
\]
The approximate level estimate thus follows from the consensus already computed and a single
scalar sum. At $\lambda_j^v=0$, $\kappa^v_{ij}/m_{v,j}=m_{v,i}/\sigma_{v,j}^2$
is proportional to $b_j=1/\sigma_{v,j}^2$ (with $m_{v,i}$ canceling between
numerator and denominator), so the replacement is exact---recovering
$\bar m_i=\sum_j\pi_{ij}m_{v,j}$ and the level of the Polar Transformer
under common noise. At $\lambda_i^v=1$, where the two weightings diverge
most, the radial step vanishes, so any inaccuracy in $\hat\ell_i$ drops out
of the update.

The resulting value update is
\[
\Delta\boldsymbol v_i
=\beta_{t,i}\tilde m_i\boldsymbol P_T(\hat{\boldsymbol u}_{v,i})\bar{\boldsymbol u}_i
+\beta_{r,i}\,\Delta m_i\,\hat{\boldsymbol u}_{v,i},
\]
where $\Delta m_i=(1-\lambda_i^v)(\hat\ell_i-m_{v,i})$ and
$\tilde m_i=\lambda_i^v(m_{v,i}+\beta_{r,i}\Delta m_i)+(1-\lambda_i^v)\bar m_i$.
The logits depend elementwise on $c_{ij}$ and the magnitudes, so the whole
layer fits one pass of a tiled kernel, at the cost of the Polar Transformer
plus one scalar sum per token (Algorithm~\ref{alg:general}).


\paragraph{Algorithm.}

Algorithm~\ref{alg:general} gives the single-head Learned-Scatter Polar
Transformer layer; here, the tangent space projection is 
$\boldsymbol P_T(\boldsymbol v)\boldsymbol w=\boldsymbol w-\boldsymbol v\,\mathrm{Re}(\boldsymbol v^\dagger\boldsymbol w)/r^2$
for $\|\boldsymbol v\|=r$.

\begin{algorithm}[H]
\caption{Learned-Scatter Polar Transformer (single head)}
\label{alg:general}
\begin{algorithmic}[1]
\REQUIRE Sequence $\boldsymbol Z\in\mathbb R^{d\times N}$ at positions $t_1,\dots,t_N$
\REQUIRE Projections $\boldsymbol W_q,\boldsymbol W_k,\boldsymbol W_v$ ($d\to2d$) and $\boldsymbol W_o$ ($2d\to d$); frequencies $\boldsymbol\omega$ in conjugate pairs
\REQUIRE Normalization gain $g$; noise variances $\eta_k^2,\rho_k^2,\eta_v^2,\rho_v^2$; scatters $\zeta_k^2,\zeta_v^2$; DoF parameter $\nu_t$; prior strengths $c_r,c_t$
\ENSURE Updated sequence $\boldsymbol Z_{\mathrm{out}}\in\mathbb R^{d\times N}$
\STATE $\boldsymbol Q,\ \boldsymbol K,\ \boldsymbol V \gets \boldsymbol W_q\boldsymbol Z,\ \boldsymbol W_k\boldsymbol Z,\ \boldsymbol W_v\boldsymbol Z$ \COMMENT{project}
\STATE $\boldsymbol m_q,\ \boldsymbol m_k,\ \boldsymbol m_v \gets$ column norms of $\boldsymbol Q,\ \boldsymbol K,\ \boldsymbol V$ \COMMENT{magnitudes}\label{line:g_magnitudes}
\STATE $\hat{\boldsymbol Q},\ \hat{\boldsymbol K},\ \hat{\boldsymbol V} \gets \mathrm{RMSNorm}_g(\boldsymbol Q,\boldsymbol K,\boldsymbol V)$ \COMMENT{columns of norm $r=g\sqrt{2d}$}
\STATE $\tilde{\boldsymbol Q},\ \tilde{\boldsymbol K},\ \tilde{\boldsymbol V} \gets \tilde{\boldsymbol\Phi}^{-}\odot(\hat{\boldsymbol Q},\hat{\boldsymbol K},\hat{\boldsymbol V})$ \COMMENT{rotate (RoPE)}
\STATE $\boldsymbol\sigma_k^2,\ \boldsymbol\sigma_v^2 \gets \eta_k^2+\rho_k^2\boldsymbol m_k^2,\ \eta_v^2+\rho_v^2\boldsymbol m_v^2,\quad \boldsymbol\lambda^k,\ \boldsymbol\lambda^v \gets \frac{\zeta_k^2}{\boldsymbol\sigma_k^2+\zeta_k^2},\ \frac{\zeta_v^2}{\boldsymbol\sigma_v^2+\zeta_v^2}$ \COMMENT{per token}
\STATE $\boldsymbol C[i,j] \gets \mathrm{Re}(\tilde{\boldsymbol Q}[:,i]^{\dagger}\tilde{\boldsymbol K}[:,j])/r^2$ \COMMENT{agreements}
\STATE $\boldsymbol m^{k*}[i,j] \gets \boldsymbol\lambda^k[j]\boldsymbol m_k[j]+(1-\boldsymbol\lambda^k[j])\boldsymbol m_q[i],\quad \boldsymbol\kappa^k[i,j] \gets \boldsymbol m_k[j]\boldsymbol m^{k*}[i,j]/\boldsymbol\sigma_k^2[j]$ \COMMENT{keys}
\STATE $\boldsymbol D^2[i,j] \gets 2\boldsymbol\kappa^k[i,j]\big(1-\boldsymbol C[i,j]\big)+\big(\boldsymbol m_k[j]-\boldsymbol m_q[i]\big)^2/\big(\boldsymbol\sigma_k^2[j]+\zeta_k^2\big)$ \COMMENT{test}\label{line:g_distance}
\STATE $\boldsymbol\Lambda \gets -(\nu_t+1)\log(1+\boldsymbol D^2/\nu_t d)+\boldsymbol M_{\mathrm{causal}}$ \COMMENT{robust log-weights, $\le0$}
\STATE $\boldsymbol W[i] \gets \sum_je^{\boldsymbol\Lambda[i,j]},\quad \boldsymbol\beta_r,\ \boldsymbol\beta_t \gets \boldsymbol W/(c_r+\boldsymbol W),\ \boldsymbol W/(c_t+\boldsymbol W)$ \COMMENT{evidence shares}
\STATE $\boldsymbol m^{v*}[i,j] \gets \boldsymbol\lambda^v[j]\boldsymbol m_v[j]+(1-\boldsymbol\lambda^v[j])\boldsymbol m_v[i],\quad \boldsymbol\kappa^v[i,j] \gets \boldsymbol m_v[j]\boldsymbol m^{v*}[i,j]/\boldsymbol\sigma_v^2[j]$ \COMMENT{values}\label{line:g_precision}
\STATE $\boldsymbol A \gets \mathrm{softmax}_j(\log\boldsymbol\kappa^v+\boldsymbol\Lambda)$ \COMMENT{attention}
\STATE $\bar{\boldsymbol V} \gets \tilde{\boldsymbol\Phi}^{+}\odot\big(\tilde{\boldsymbol V}\boldsymbol A^{\top}\big)$ \COMMENT{consensus}\label{line:g_consensus}
\STATE $\bar{\boldsymbol m}[i] \gets \big(\sum_j\boldsymbol A[i,j]/\boldsymbol m_v[j]\big)^{-1},\quad \hat{\boldsymbol\ell}[i] \gets \bar{\boldsymbol m}[i]\,\mathrm{Re}\big(\hat{\boldsymbol V}[:,i]^{\dagger}\bar{\boldsymbol V}[:,i]\big)/r^2$ \COMMENT{level}\label{line:g_level}
\STATE $\Delta\boldsymbol m \gets (1-\boldsymbol\lambda^v)\odot(\hat{\boldsymbol\ell}-\boldsymbol m_v),\quad \tilde{\boldsymbol m}\gets\boldsymbol\lambda^v\odot(\boldsymbol m_v+\boldsymbol\beta_r\odot\Delta\boldsymbol m)+(1-\boldsymbol\lambda^v)\odot\bar{\boldsymbol m}$ \COMMENT{radial step}
\STATE $\Delta\boldsymbol V[:,i] \gets \boldsymbol\beta_t[i]\,\tilde{\boldsymbol m}[i]\,\boldsymbol P_T(\hat{\boldsymbol V}[:,i])\,\bar{\boldsymbol V}[:,i]/r+\boldsymbol\beta_r[i]\,\Delta\boldsymbol m[i]\,\hat{\boldsymbol V}[:,i]/r$ \COMMENT{proximal step}\label{line:g_update}
\STATE $\boldsymbol Z^{+} \gets \boldsymbol Z+\boldsymbol W_o\Delta\boldsymbol V$ \COMMENT{residual update}
\STATE $\boldsymbol Z_{\mathrm{out}} \gets \boldsymbol Z^{+}+\mathrm{FFN}\big(\mathrm{RMSNorm}(\boldsymbol Z^{+})\big)$ \COMMENT{feedforward}
\RETURN $\boldsymbol Z_{\mathrm{out}}$
\end{algorithmic}
\end{algorithm}

\noindent\emph{Signed variant.} Line~\ref{line:g_distance} uses the
marginal distance,
$\boldsymbol m_k[j]^2(1-\boldsymbol C[i,j]^2)/\boldsymbol\sigma_k^2[j]+(\boldsymbol m_k[j]\boldsymbol C[i,j]-\boldsymbol m_q[i])^2/(\boldsymbol\sigma_k^2[j]+\zeta_k^2)$,
and line~\ref{line:g_consensus} aggregates with $\boldsymbol A\odot\boldsymbol G$,
where
$\boldsymbol G[i,j]=\big(\boldsymbol\lambda^v[j]\boldsymbol m_v[j]\boldsymbol C[i,j]+(1-\boldsymbol\lambda^v[j])\boldsymbol m_v[i]\big)/\boldsymbol m^{v*}[i,j]$
is the signed gate.

\newpage

\subsection{Signed Amplitudes and Power Iteration}

Standard attention uses nonnegative weights, restricting updates to convex
combinations: tokens at opposite ends of a bipolar feature---such as opposing
polarities along a semantic axis or antipodal readings of a principal
component---either cancel or reject one another as outliers. Allowing a
\emph{signed amplitude} $m_j\in\mathbb R$ resolves this: $\pm\boldsymbol u$
represent the same axis, and the directional distribution is Watson
\citep{mardia2000directional}. Because a reversed reading is explained by a
negative amplitude rather than rejected, we replace the plug-in vMF with the
exact marginal likelihood.

Integrating out $m_j\sim\mathcal N(\ell,\zeta^2)$, the reading
$\hat{\boldsymbol x}_j$ is Gaussian with mean $\ell\boldsymbol u$ and
variance $\sigma_j^2+\zeta^2$ along $\boldsymbol u$ and $\sigma_j^2$ across
it, giving the squared Mahalanobis distance
\[
d_j^2=\frac{\hat m_j^2(1-c_j^2)}{\sigma_j^2}+\frac{(\hat m_jc_j-\ell)^2}{\sigma_j^2+\zeta^2}
=\frac{1}{\sigma_j^2} \Big( \hat m_j^2-\lambda_j\hat m_j^2c_j^2-2(1-\lambda_j)\hat m_j\ell c_j+(1-\lambda_j)\ell^2\Big).
\]
Its directional part has a $c_j^2$ term, which treats $\pm\boldsymbol u$
alike, and a linear $c_j$ term, which penalizes reversed readings in
proportion to $(1-\lambda_j)\ell$. At $\lambda_j=0$, the $c_j^2$ terms cancel
and $d_j^2$ recovers Eq.~\ref{eq:rt_distance}; at $\lambda_j=1$, it is
$\hat m_j^2(1-c_j^2)/\sigma_j^2$, depending only on the axis.

Minimizing $\tfrac12\sum_j w_j d_j^2$ in $\boldsymbol u$ yields the update.
Differentiating $\tfrac12 d_j^2$ via $\partial c_j/\partial\boldsymbol u=\hat{\boldsymbol u}_j$
gives a negative gradient of $(\hat m_j m_j^\circ/\sigma_j^2)\hat{\boldsymbol u}_j$,
with posterior amplitude $m_j^\circ=\lambda_j\hat m_j c_j+(1-\lambda_j)\ell$.
In terms of $\kappa_j=\hat m_j m_j^*/\sigma_j^2$, each reading pulls the
direction by $\kappa_j g_j \hat{\boldsymbol u}_j$ with signed gate
$g_j=m_j^\circ/m_j^*$; when $m_j^\circ<0$, the gate flips anti-aligned
readings so $+\boldsymbol u$ and $-\boldsymbol u$ reinforce.

Because $d_j^2$ is quadratic in $c_j$, the optimal direction has no closed
form; instead, we take one projected gradient step on
$\tfrac12\sum_j w_j d_j^2$ at $\hat{\boldsymbol u}_i$, replacing the
consensus with $\bar{\boldsymbol u}_i=\sum_j A_{ij} g_j \hat{\boldsymbol u}_j$x. At $\lambda_j=1$ (or $\ell=0$), $g_j=c_j$, so with unit robust weights, the
negative gradient is
\[
\sum_j\kappa_j c_j\hat{\boldsymbol u}_j
=\boldsymbol S\hat{\boldsymbol u}_i,
\qquad
\boldsymbol S=\sum_j\kappa_j\hat{\boldsymbol u}_j\hat{\boldsymbol u}_j^\top,
\]
whose fixed point is the dominant eigenvector of $\boldsymbol S$. With
$T=\operatorname{tr}\boldsymbol S$ and
$s=\hat{\boldsymbol u}_i^\top\boldsymbol S\hat{\boldsymbol u}_i$,
$\bar{\boldsymbol u}_i=\boldsymbol S\hat{\boldsymbol u}_i/T$, so
$\hat{\boldsymbol u}_i+\Delta\boldsymbol u_i=[\boldsymbol S+(T-s)\boldsymbol I]\hat{\boldsymbol u}_i/T$
is one shifted power-iteration step, preserving eigenvectors
\citep{golub2013matrix}. With separate $Q,K,V$, $c_{ij}$ replaces value
agreement and $g_j$ scales $\boldsymbol A$ outside the softmax, staying
single-pass.

Whereas iterated robust reweighting performs spherical mean shift
\citep{cheng1995mean,kobayashi2010vmf} to a query-dependent local mode, power
iteration extracts the dominant axis without a bandwidth and pools antipodal
clusters that unsigned methods cancel or reject
(Figure~\ref{fig:watson_modes}), trading away query-specific retrieval.
Because the learned scatter $\zeta^2$ sets $\lambda_j$, each head can drive
$\lambda_j\to0$ for query-specific consensus or $\lambda_j\to1$ to pool
bipolar evidence along the dominant axis.

\begin{figure}[H]
\centering
\includegraphics[
    width=0.45\linewidth,
    trim={0.0cm 0.0cm 0.0cm 0.25cm},
    clip
]{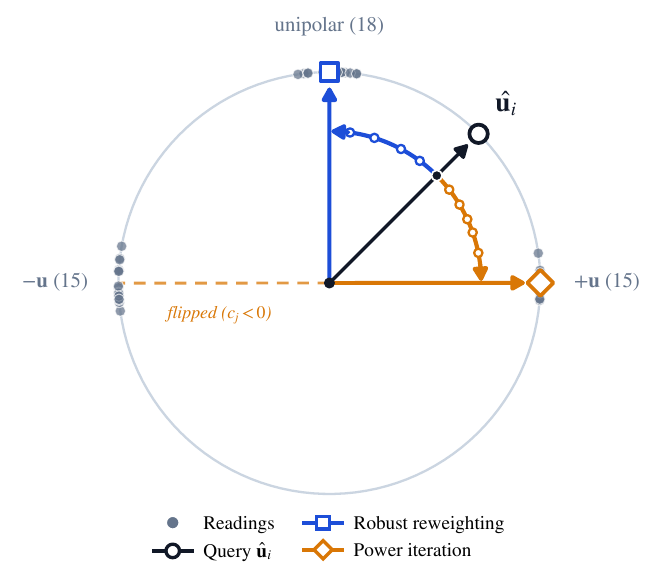}
\caption{\textbf{Signed versus unsigned mode selection.} A bipolar axis
($15$ readings each near $\pm\boldsymbol u$) competes with an orthogonal
unipolar cluster ($18$ readings) from an equidistant query
$\hat{\boldsymbol u}_i$ (black). Unsigned attention (blue) rejects
$-\boldsymbol u$ and converges to the unipolar cluster, whereas shifted power
iteration (gold) flips $-\boldsymbol u$ via negative gates $c_j<0$ (dashed),
pooling both halves onto $+\boldsymbol u$.}
\label{fig:watson_modes}
\end{figure}

\end{document}